\documentclass[11pt]{article}
\usepackage[letterpaper,margin=1.0in]{geometry}
\usepackage{enumitem}
\usepackage{setspace}
\usepackage{microtype}
\usepackage{booktabs}
\usepackage{amsfonts,amsmath,amssymb,amsthm}
\usepackage{xcolor}
\usepackage{url}
\usepackage{subcaption}
\usepackage{graphicx}
\usepackage{hyperref}
\hypersetup{hidelinks}
\usepackage{mathtools}
\usepackage{bbm}
\usepackage{mathrsfs}
\usepackage{nicefrac}
\usepackage{algorithm}
\usepackage{algorithmic}
\usepackage{float}
\usepackage{comment}
\newtheorem{theorem}{Theorem}[section]
\newtheorem{lemma}[theorem]{Lemma}

\newtheorem{corollary}[theorem]{Corollary}
\newtheorem{definition}[theorem]{Definition}
\newtheorem{assumption}[theorem]{Assumption}
\newtheorem{proposition}[theorem]{Proposition}

\newtheorem{example}[theorem]{Example}
\newtheorem{remark}[theorem]{Remark}
\newcommand{\cA}{\mathcal{A}}
\newcommand{\cB}{\mathcal{B}}

\newcommand{\cD}{\mathcal{D}}
\newcommand{\cF}{\mathcal{F}}

\newcommand{\cL}{\mathcal{L}}
\newcommand{\cM}{\mathcal{M}}

\newcommand{\cO}{\mathcal{O}}

\newcommand{\cQ}{\mathcal{Q}}
\newcommand{\cR}{\mathcal{R}}
\newcommand{\cS}{\mathcal{S}}

\newcommand{\cX}{\mathcal{X}}
\newcommand{\EE}{\mathbb{E}}
\newcommand{\PP}{\mathbb{P}}
\newcommand{\R}{\mathbb{R}}

\newcommand{\one}{\mathbf{1}}
\newcommand{\eps}{\varepsilon}

\newcommand{\KL}{\operatorname{KL}}

\newcommand{\LSE}{\operatorname{LSE}}
\newcommand{\Var}{\operatorname{Var}}

\newcommand{\argmax}{\operatorname*{arg\,max}}
\newcommand{\argmin}{\operatorname*{arg\,min}}

\providecommand{\cJ}{\mathcal{J}}
\providecommand{\cT}{\mathcal{T}}

\providecommand{\iid}{\overset{\mathrm{i.i.d.}}{\sim}}

\providecommand{\diag}{\operatorname{diag}}

\providecommand{\Ent}{\operatorname{H}}

\usepackage[
  backend=biber,
  style=authoryear,
  natbib=true,
  maxcitenames=2,
  maxbibnames=99,
  uniquename=false,
  giveninits=true
]{biblatex}
\DeclareDelimFormat{nameyeardelim}{\addcomma\space}
\title{A Lecture Note on Offline RL and IRL\\
\large Part II: Foundations of Inverse Reinforcement Learning\\ and Dynamic Discrete Choice Models}
\author{Enoch Hyunwook Kang\\
University of Washington, Foster School of Business}
\date{\today}
\begin{document}
\maketitle
\begin{abstract}
In the forward reinforcement-learning problem, the reward is fixed and known; the learner is asked to find a good policy or value function. Here we turn the question around. Given offline data generated by an expert, can we recover the \emph{reward} the expert was optimizing? This is the \emph{inverse} reinforcement learning problem, and remarkably, two communities, structural econometricians studying \emph{dynamic discrete choice} (DDC) and machine learners studying \emph{entropy-regularized IRL}, have been working on exactly the same probabilistic model under different names. We begin by proving their equivalence. We then develop the classical identification result of \citet{magnac2002identifying} and the classical computational paradigms that grew out of it: Rust's nested fixed-point algorithm \citep{rust1987optimal}, the conditional-choice-probability approach of \citet{hotz1993conditional}, and the two temporal-difference approaches of \citet{adusumilli2019temporal}: linear semi-gradient TD and approximate value iteration. Each route has its limits: dimensionality, transition-kernel estimation, the deadly triad, or projected fixed-point bias. We then walk through the modern ML/IRL strand: adversarial IRL \citep{fu2017learning}, occupancy matching \citep{ho2016generative}, IQ-Learn \citep{garg2021iq}, and offline ML-IRL \citep{zeng2023understanding}, deriving each method's actual objective and stating precisely what it does and does not identify. We close with the empirical-risk-minimization framework of \citet{kang2025empirical}, which yields a gradient-based estimator for offline IRL/DDC. 
\end{abstract}

\newpage
\tableofcontents
\newpage

\section{Why IRL and DDC?}
\label{sec:why}

In the forward reinforcement-learning problem, the reward function is a fixed, known input. The world handed us $r$ and asked us to learn $Q^\star$. In a great many real-world problems, however, $r$ is exactly what we want to learn, and the closest thing we have to it is a log of decisions made by someone who knew what they were doing.

\begin{itemize}[leftmargin=2em]
\item A labor economist watches workers choose between job offers and wants to recover the utility function that rationalizes their choices \citep{mcfadden2001economic}.
\item A marketing analyst sees customers' coupon redemptions and purchase histories, and wants to recover preferences in order to predict response to a new offer.
\item A health-economics researcher observes physicians' treatment decisions and wants to quantify the implicit risk-benefit tradeoffs they are making.
\item A robotics or autonomous-driving engineer watches expert pilots fly a helicopter or expert drivers navigate traffic, and wants to extract a reward function that, when fed to a downstream RL algorithm, reproduces the expert's behavior \citep{barnes2023massively}.
\end{itemize}
In each case the data is offline: a fixed log of state-action-(next-state) tuples from one or more experts, and the reward itself is never directly observed \citep{levine2020offline}.

Two literatures have been chipping away at this problem for decades, under almost identical probabilistic structure but completely different vocabularies:
\begin{itemize}[leftmargin=2em]
\item \emph{Dynamic discrete choice} (DDC) models from structural econometrics \citep{rust1987optimal,rust1994structural,arcidiacono2011practical}: the agent's reward is observed plus a private idiosyncratic shock, generally Gumbel-distributed, which the econometrician integrates out.
\item \emph{Inverse reinforcement learning} (IRL) with the \emph{entropy-regularized} (maximum-entropy) variant of \citet{finn2016connection,fu2017learning}: the agent maximizes expected return plus an entropy bonus on its policy, producing the same softmax choice rule.
\end{itemize}
Our first order of business (Section~\ref{sec:prelim}) is to prove these are the same model. Sections~\ref{sec:identification}--\ref{sec:irl} then survey the methods each tradition developed, and Section~\ref{sec:erm} presents the unified ERM treatment of \citet{kang2025empirical} that fixes the issues both traditions inherited.

\section{Preliminaries}
\label{sec:prelim}

This section sets up the model, defines the soft Bellman operator, proves its contraction, and, most importantly, proves the equivalence between the DDC and entropy-regularized IRL formulations. Readers familiar with standard MDP notation can skim the MDP setup, but should not skip Section~\ref{sec:soft-policy} (the Gumbel-max identity) or Section~\ref{sec:equivalence} (the equivalence theorem).

\subsection{Markov Decision Processes}
\label{sec:mdp}

We consider an infinite-horizon discounted MDP $\cM = (\cS, \cA, P, \nu_0, r, \beta)$ where $\cS$ is a (possibly continuous) state space, $\cA$ is a \emph{finite} action space, $P : \cS \times \cA \to \Delta(\cS)$ is the transition kernel, $\nu_0 \in \Delta(\cS)$ is the initial distribution, $r : \cS \times \cA \to [-R_{\max}, R_{\max}]$ is the (unknown) reward, and $\beta \in (0, 1)$ is the discount. A stationary Markov policy is a measurable map $\pi : \cS \to \Delta(\cA)$ with $\pi(a \mid s)$ the probability of $a$ at $s$. We write $\PP_{\nu_0, \pi}$ and $\EE_{\nu_0, \pi}$ for the distribution and expectation over the resulting trajectory $\tau = (s_0, a_0, s_1, a_1, \ldots)$.

The discounted state-action occupancy of $\pi$ is
\begin{equation*}
d^\pi(s, a) \;:=\; (1 - \beta) \sum_{t=0}^\infty \beta^t\, \PP_{\nu_0, \pi}(s_t = s,\, a_t = a),
\end{equation*}
a probability measure on $\cS \times \cA$. The offline dataset is
\begin{equation*}
\cD \;=\; \bigl\{(s_i, a_i, s_i')\bigr\}_{i=1}^N,
\end{equation*}
sampled iid from $d^{\pi^\star}$ for a single expert policy $\pi^\star$. Note what is missing: \emph{we do not observe rewards in $\cD$}: only states, actions, and next states. That omission is precisely what makes the problem ``inverse,'' and it is the canonical setup of both DDC and IRL.

\subsection{Function Spaces, Norms, Projections}
\label{sec:norms}

Let $\cQ := \{Q : \cS \times \cA \to \R\;\text{measurable, bounded}\}$. The sup-norm is $\|Q\|_\infty := \sup_{s, a} |Q(s,a)|$. For $\rho \in \Delta(\cS \times \cA)$ we write $\|Q\|_\rho^2 := \EE_{(s,a) \sim \rho}[Q(s,a)^2]$ and $\langle Q_1, Q_2 \rangle_\rho := \EE_\rho[Q_1 Q_2]$, so $L^2(\rho)$ is a Hilbert space. For a closed linear subspace $\cF \subseteq L^2(\rho)$, the $L^2(\rho)$-projection is the unique element $\Pi_{\cF, \rho} Q \in \cF$ satisfying $\|Q-\Pi_{\cF, \rho}Q\|_\rho = \inf_{f \in \cF}\|Q-f\|_\rho$. This restriction to closed linear subspaces matters: projections onto arbitrary closed nonconvex sets may fail to be unique. For closed linear subspaces, $\Pi_{\cF, \rho}$ is a non-expansion: $\|\Pi Q_1 - \Pi Q_2\|_\rho \le \|Q_1 - Q_2\|_\rho$ for all $Q_1, Q_2$. We will need this fact, both the projection itself and its non-expansion property, when analyzing temporal-difference methods in Section~\ref{sec:td-ddc}.

\subsection{The Dynamic Discrete Choice Formulation}
\label{sec:ddc}

The DDC setup follows the Type 1 Extreme Value (T1EV) derivation in \citet{kang2025empirical}. At each period the agent observes the deterministic flow payoff $r(s_t,a_t)$ plus an idiosyncratic alternative-specific shock $\epsilon_{t,a_t}$. We use the unit-scale Gumbel distribution with location parameter $\delta$:
\begin{align}
F_\delta(x)&=\exp\!\bigl(-\exp(-(x-\delta))\bigr),\label{eq:gumbel-cdf}\\
f_\delta(x)&=\exp\!\bigl(-((x-\delta)+\exp(-(x-\delta)))\bigr),\label{eq:gumbel-pdf}\\
\EE[\epsilon]&=\delta+\gamma_E=:\mu_\epsilon,
\label{eq:gumbel-location-one}
\end{align}
where $\gamma_E$ is the Euler--Mascheroni constant. The mean-zero T1EV normalization used for the baseline DDC--MaxEnt equivalence is
\begin{equation*}
\delta=-\gamma_E,
\qquad
\mu_\epsilon=0.
\end{equation*}

The DDC agent maximizes
\begin{equation}
\pi^\star
\in
\argmax_{\pi}\;
\EE_{\nu_0,\pi,G}\!\left[
\sum_{t=0}^\infty \beta^t\bigl(r(s_t,a_t)+\epsilon_{t,a_t}\bigr)
\right],
\label{eq:ddc-objective}
\end{equation}
where the expectation averages over the Markov chain and the idiosyncratic shocks. The policy is induced by choosing the alternative with the largest shock-augmented choice-specific value. The derivation in Section~\ref{sec:soft-policy} shows that, under the mean-zero normalization $\delta=-\gamma_E$, this DDC problem implies
\begin{align}
\pi^\star(a\mid s)
&=\frac{\exp(Q^\star(s,a))}{\sum_{a'\in\cA}\exp(Q^\star(s,a'))},
\label{eq:opt-policy}\\
Q^\star(s,a)
&=r(s,a)+\beta\,\EE_{s'\sim P(\cdot\mid s,a)}\!\left[
\log\sum_{a'\in\cA}\exp(Q^\star(s',a'))
\right].
\label{eq:soft-bellman}
\end{align}
If $\delta\neq-\gamma_E$, the same derivation gives the same softmax rule but adds the constant $\beta(\delta+\gamma_E)$ to the right side of the $Q$ Bellman equation.

\subsection{The Entropy-Regularized IRL Formulation}
\label{sec:irl-setup}

The MaxEnt-IRL setup in \citet{kang2025empirical} instead treats the agent as choosing a stochastic stationary policy and explicitly rewards entropy. With regularization coefficient $\lambda$, the forward problem is
\begin{equation}
\pi^\star
\in
\argmax_{\pi}\;
\EE_{\nu_0,\pi}\!\left[
\sum_{t=0}^\infty \beta^t\bigl(r(s_t,a_t)+\lambda\Ent(\pi(\cdot\mid s_t))\bigr)
\right],
\label{eq:irl-objective}
\end{equation}
where $\Ent(\pi(\cdot\mid s)):=-\sum_a\pi(a\mid s)\log\pi(a\mid s)$. Throughout the main text, and in the equivalence theorem below, we set
\begin{equation*}
\lambda=1,
\end{equation*}
which is the entropy scale corresponding to unit-scale T1EV shocks. Section~\ref{sec:soft-policy} derives the same softmax rule~\eqref{eq:opt-policy} and soft Bellman equation~\eqref{eq:soft-bellman} from this entropy-regularized problem.

\subsection{The Soft Bellman Operator}
\label{sec:soft-bellman}

\begin{definition}[Soft Bellman operator]
\label{def:soft-bellman}
For $Q \in \cQ$, define
\begin{equation*}
(\cT_{\mathrm{soft}} Q)(s, a) \;:=\; r(s, a) + \beta\, \EE_{s' \sim P(\cdot \mid s, a)}\!\Bigl[\log \sum_{a' \in \cA} \exp(Q(s', a'))\Bigr].
\end{equation*}
We write $\LSE(Q(s', \cdot)) := \log \sum_{a' \in \cA} \exp(Q(s', a'))$ for the unit-temperature log-sum-exp. If the DDC shock mean is $\mu_\epsilon\neq 0$, the corresponding uncentered DDC operator is
\begin{equation*}
(\cT_{\mathrm{DDC},\mu_\epsilon}Q)(s,a)
\;:=\;
(\cT_{\mathrm{soft}}Q)(s,a)+\beta\mu_\epsilon.
\end{equation*}
\end{definition}

The log-sum-exp is a smooth surrogate for $\max$, sandwiched between it and a $\log|\cA|$ slack: $\max_{a'} Q(s', a') \le \LSE(Q(s', \cdot)) \le \max_{a'} Q(s', a') + \log|\cA|$. The two facts below are what we need to treat $\cT_{\mathrm{soft}}$ as well-behaved as the ordinary hard Bellman operator.

\begin{lemma}[Log-sum-exp is 1-Lipschitz in sup-norm]
\label{lem:lse-lip}
For bounded $f, g : \cA \to \R$,
\begin{equation*}
\bigl|\LSE(f) - \LSE(g)\bigr| \;\le\; \max_{a \in \cA} |f(a) - g(a)|.
\end{equation*}
\end{lemma}

\begin{proof}
Let $\Delta := \max_a |f(a) - g(a)|$, so $g(a) - \Delta \le f(a) \le g(a) + \Delta$ for every $a$. Exponentiating preserves the inequalities and gives
\begin{equation*}
e^{-\Delta} \sum_a e^{g(a)} \;\le\; \sum_a e^{f(a)} \;\le\; e^{\Delta} \sum_a e^{g(a)}.
\end{equation*}
Taking logs, $\LSE(g) - \Delta \le \LSE(f) \le \LSE(g) + \Delta$, hence $|\LSE(f) - \LSE(g)| \le \Delta$.
\end{proof}

\begin{lemma}[Shift identity for LSE]
\label{lem:lse-shift}
For bounded $f : \cA \to \R$ and any $c \in \R$,
\begin{equation*}
\LSE(f + c) \;=\; \LSE(f) + c, \qquad \text{where } (f + c)(a) := f(a) + c.
\end{equation*}
\end{lemma}

\begin{proof}
$\LSE(f + c) = \log \sum_a e^{f(a) + c} = \log\!\bigl(e^c \sum_a e^{f(a)}\bigr) = c + \LSE(f)$.
\end{proof}

\begin{lemma}[Sup-norm contraction of $\cT_{\mathrm{soft}}$]
\label{lem:soft-contraction}
For any $Q_1, Q_2 \in \cQ$,
\begin{equation*}
\|\cT_{\mathrm{soft}} Q_1 - \cT_{\mathrm{soft}} Q_2\|_\infty \;\le\; \beta\,\|Q_1 - Q_2\|_\infty.
\end{equation*}
\end{lemma}

\begin{proof}
The reward terms cancel. For any fixed $(s, a)$,
\begin{align*}
|\cT_{\mathrm{soft}} Q_1(s,a) - \cT_{\mathrm{soft}} Q_2(s,a)|
&= \beta\,\bigl|\EE_{s' \sim P(\cdot \mid s, a)}[\LSE(Q_1(s',\cdot)) - \LSE(Q_2(s',\cdot))]\bigr|\\
&\le \beta\,\EE_{s'}\bigl[\bigl|\LSE(Q_1(s',\cdot)) - \LSE(Q_2(s',\cdot))\bigr|\bigr] && \text{(Jensen)}\\
&\le \beta\,\EE_{s'}\bigl[\max_{a'} |Q_1(s',a') - Q_2(s',a')|\bigr] && \text{(Lemma~\ref{lem:lse-lip})}\\
&\le \beta\,\|Q_1 - Q_2\|_\infty.
\end{align*}
Taking sup over $(s,a)$ closes the bound.
\end{proof}

By Banach's fixed-point theorem on $(\cQ, \|\cdot\|_\infty)$, $\cT_{\mathrm{soft}}$ has a unique fixed point $Q^\star$ and \emph{soft value iteration} $Q_{k+1} := \cT_{\mathrm{soft}} Q_k$ converges geometrically: $\|Q_k - Q^\star\|_\infty \le \beta^k \|Q_0 - Q^\star\|_\infty$. So far this is identical to the hard-Bellman story; the action is in what comes next.

\subsection{DDC--MaxEnt-IRL Equivalence}
\label{sec:soft-policy}

\begin{lemma}[T1EV maximum, choice probability, and selected-value identities]
\label{lem:gumbel-max}
Let $G_1,\ldots,G_K$ be independent with $G_i\sim\operatorname{Gumbel}(\nu_i,1)$ and define $S:=\sum_i e^{\nu_i}$. Then
\begin{align}
\max_i G_i &\sim \operatorname{Gumbel}(\log S,1),
\label{eq:gumbel-max-distribution}\\
\EE\bigl[\max_i G_i\bigr] &= \gamma_E+\log S,
\label{eq:gumbel-LSE}\\
\PP\{k=\argmax_i G_i\} &= \frac{e^{\nu_k}}{S},
\label{eq:gumbel-softmax}\\
\EE\bigl[G_k\mid k=\argmax_i G_i\bigr]
&=\gamma_E+\nu_k-\log\frac{e^{\nu_k}}{S}.
\label{eq:gumbel-selected-value}
\end{align}
\end{lemma}

\begin{proof}
For the maximum, independence gives
\begin{align*}
\PP\{\max_i G_i\le x\}
&=\prod_i \exp\{-\exp(-(x-\nu_i))\} \\
&=\exp\left\{-e^{-x}\sum_i e^{\nu_i}\right\}
=\exp\{-\exp(-(x-\log S))\},
\end{align*}
so $\max_iG_i\sim\operatorname{Gumbel}(\log S,1)$ and its expectation is $\gamma_E+\log S$. The choice probability follows from the standard T1EV integral:
\begin{align*}
\PP\{G_k>\max_{i\neq k}G_i\}
&=\int_{-\infty}^\infty
\prod_{i\neq k}F_i(x)f_k(x)\,dx
=\frac{e^{\nu_k}}{\sum_i e^{\nu_i}}.
\end{align*}
For the selected-value identity, let $M_{-k}:=\max_{i\neq k}G_i$. The first part shows $M_{-k}\sim\operatorname{Gumbel}(\log\sum_{i\neq k}e^{\nu_i},1)$. Applying the two-variable T1EV conditional-expectation calculation to $G_k$ and $M_{-k}$ yields
\begin{equation*}
\EE[G_k\mid G_k\ge M_{-k}]
=\gamma_E+\nu_k+\log\left(1+\frac{\sum_{i\neq k}e^{\nu_i}}{e^{\nu_k}}\right)
=\gamma_E+\nu_k-\log\frac{e^{\nu_k}}{S},
\end{equation*}
which is exactly~\eqref{eq:gumbel-selected-value}.
\end{proof}

\begin{lemma}[Entropy-regularized discrete choice identity]
\label{lem:entropy-choice}
For $x=(x_a)_{a\in\cA}$ and $\lambda=1$,
\begin{align}
\max_{q\in\Delta(\cA)}\left\{\sum_a q_a x_a + \Ent(q)\right\}
&=\log\sum_a e^{x_a},
\label{eq:entropy-lse}\\
\argmax_{q\in\Delta(\cA)}\left\{\sum_a q_a x_a + \Ent(q)\right\}_a
&=\frac{e^{x_a}}{\sum_b e^{x_b}}.
\label{eq:entropy-softmax}
\end{align}
\end{lemma}

\begin{proof}
The objective is strictly concave on the simplex. The Lagrangian first-order condition is
\begin{equation*}
x_a-\log q_a-1-\eta=0,
\end{equation*}
so $q_a\propto e^{x_a}$ and normalization gives~\eqref{eq:entropy-softmax}. Substitution gives~\eqref{eq:entropy-lse}.
\end{proof}

\paragraph{MaxEnt-IRL derivation.}
Define the value and choice-specific value functions by
\begin{align*}
V^\star(s)
&:=\EE_{\pi^\star}\!\left[\sum_{h=0}^\infty\beta^h
\bigl(r(s_h,a_h)+\Ent(\pi^\star(\cdot\mid s_h))\bigr)\;\middle|\;s_0=s\right],\\
Q^\star(s,a)&:=r(s,a)+\beta\,\EE_{s'\sim P(\cdot\mid s,a)}[V^\star(s')].
\end{align*}
At a fixed state $s$, the Bellman recursion is the finite-dimensional entropy-regularized choice problem
\begin{align}
V^\star(s)
&=\max_{q\in\Delta(\cA)}
\left\{
\sum_{a\in\cA}q_a\left(r(s,a)+\beta\EE[V^\star(s')\mid s,a]\right)+\Ent(q)
\right\} \notag\\
&=\max_{q\in\Delta(\cA)}
\left\{\sum_{a\in\cA}q_aQ^\star(s,a)+\Ent(q)\right\}.
\label{eq:irl-bellman-v-choice}
\end{align}
By Lemma~\ref{lem:entropy-choice},
\begin{align}
V^\star(s)&=\log\sum_{a\in\cA}\exp(Q^\star(s,a)),
\label{eq:irl-value-lse}\\
\pi^\star(a\mid s)&=\frac{\exp(Q^\star(s,a))}{\sum_{a'\in\cA}\exp(Q^\star(s,a'))}.
\label{eq:irl-policy-softmax}
\end{align}
Substituting~\eqref{eq:irl-value-lse} into the definition of $Q^\star$ gives the choice-specific Bellman equation
\begin{equation}
Q^\star(s,a)
=r(s,a)+\beta\,\EE_{s'\sim P(\cdot\mid s,a)}\!\left[
\log\sum_{a'\in\cA}\exp(Q^\star(s',a'))
\right].
\label{eq:irl-q-lse-bellman}
\end{equation}
The paper also records the equivalent continuation representation. Since~\eqref{eq:irl-policy-softmax} implies
\begin{equation*}
Q^\star(s',a')-\log\pi^\star(a'\mid s')
=\log\sum_{b\in\cA}\exp(Q^\star(s',b))
\qquad \text{for every }a',
\end{equation*}
Equation~\eqref{eq:irl-q-lse-bellman} can be written as
\begin{equation}
Q^\star(s,a)
=r(s,a)+\beta\,\EE_{s'\sim P(\cdot\mid s,a),\,a'\sim\pi^\star(\cdot\mid s')}\!\left[
Q^\star(s',a')-\log\pi^\star(a'\mid s')
\right].
\label{eq:irl-q-policy-continuation}
\end{equation}
This is the analogue of Equation~(A5) in the paper.

\paragraph{DDC derivation.}
Let $\epsilon_{t,a}\iid\operatorname{Gumbel}(\delta,1)$ across actions and dates. The shock-realized value and its shock average are
\begin{align*}
V^\star(s,\epsilon)
&:=\max_{a\in\cA}\left\{r(s,a)+\epsilon_a+\beta\,\EE[\bar V^\star(s')\mid s,a]\right\},\\
\bar V^\star(s)&:=\EE_\epsilon[V^\star(s,\epsilon)],
\end{align*}
and the DDC choice-specific value is
\begin{equation}
Q^\star(s,a):=r(s,a)+\beta\,\EE_{s'\sim P(\cdot\mid s,a)}[\bar V^\star(s')].
\label{eq:ddc-q-definition-paper-style}
\end{equation}
Thus
\begin{equation}
V^\star(s,\epsilon)=\max_{a\in\cA}\{Q^\star(s,a)+\epsilon_a\}.
\label{eq:ddc-shock-realized-value}
\end{equation}
For each action $a$, $Q^\star(s,a)+\epsilon_a$ is $\operatorname{Gumbel}(Q^\star(s,a)+\delta,1)$. Lemma~\ref{lem:gumbel-max} gives
\begin{align}
\bar V^\star(s)
&=\delta+\gamma_E+\log\sum_{a\in\cA}\exp(Q^\star(s,a)),
\label{eq:ddc-value-lse-delta}\\
\pi^\star(a\mid s)
&=\PP\left\{a=\argmax_{b\in\cA}(Q^\star(s,b)+\epsilon_b)\right\}
=\frac{\exp(Q^\star(s,a))}{\sum_{a'\in\cA}\exp(Q^\star(s,a'))}.
\label{eq:ddc-policy-softmax-paper-style}
\end{align}
Substituting~\eqref{eq:ddc-value-lse-delta} into~\eqref{eq:ddc-q-definition-paper-style} yields
\begin{equation}
Q^\star(s,a)
=r(s,a)+\beta\,\EE_{s'\sim P(\cdot\mid s,a)}\!\left[
\log\sum_{a'\in\cA}\exp(Q^\star(s',a'))+\delta+\gamma_E
\right].
\label{eq:ddc-q-bellman-delta}
\end{equation}
Under the mean-zero T1EV normalization $\delta=-\gamma_E$, this becomes exactly~\eqref{eq:soft-bellman}. This is Equation~(3) in the paper.

The same DDC Bellman equation can also be written in the policy-continuation form used in the paper. Conditional on the selected next action $a'$, Lemma~\ref{lem:gumbel-max} gives
\begin{equation*}
\EE\left[Q^\star(s',a')+\epsilon_{a'}\;\middle|\;a'=\argmax_b(Q^\star(s',b)+\epsilon_b)\right]
=Q^\star(s',a')+\delta+\gamma_E-\log\pi^\star(a'\mid s').
\end{equation*}
Therefore
\begin{equation}
Q^\star(s,a)
=r(s,a)+\beta\,\EE_{s'\sim P(\cdot\mid s,a),\,a'\sim\pi^\star(\cdot\mid s')}\!\left[
Q^\star(s',a')+\delta+\gamma_E-\log\pi^\star(a'\mid s')
\right].
\label{eq:ddc-q-policy-continuation}
\end{equation}
When $\delta=-\gamma_E$, Equation~\eqref{eq:ddc-q-policy-continuation} is identical to the MaxEnt-IRL continuation representation~\eqref{eq:irl-q-policy-continuation}.

\subsection{The DDC--IRL Equivalence Theorem}
\label{sec:equivalence}

\begin{theorem}[Equivalence of mean-zero unit-scale DDC and unit-entropy MaxEnt-IRL]
\label{thm:equivalence}
Fix $\cS$, $\cA$, $P$, $\nu_0$, $r$, and $\beta$. Let the DDC shocks be iid $\operatorname{Gumbel}(\delta,1)$ with unit scale, and set $\delta=-\gamma_E$. Let the MaxEnt-IRL regularization coefficient be $\lambda=1$. Then DDC and MaxEnt-IRL are characterized by the same three equations:
\begin{align*}
Q^\star(s,a)
&=r(s,a)+\beta\,\EE_{s'\sim P(\cdot\mid s,a)}\!\left[
\log\sum_{a'\in\cA}\exp(Q^\star(s',a'))
\right],\\
\pi^\star(a\mid s)
&=\frac{\exp(Q^\star(s,a))}{\sum_{a'\in\cA}\exp(Q^\star(s,a'))},\\
Q^\star(s,a)
&=r(s,a)+\beta\,\EE_{s'\sim P(\cdot\mid s,a),\,a'\sim\pi^\star(\cdot\mid s')}\!\left[
Q^\star(s',a')-\log\pi^\star(a'\mid s')
\right].
\end{align*}
Consequently, the same reward function $r$ and discount factor $\beta$ produce the same unique choice-specific value function $Q^\star$, the same optimal policy $\pi^\star$, and hence the same joint distribution over observed $(s_t,a_t,s_{t+1})$ trajectories under any common $\nu_0$.
\end{theorem}

\begin{proof}
The MaxEnt-IRL derivation gives the three displayed equations through Lemma~\ref{lem:entropy-choice}. The DDC derivation gives the same equations through Lemma~\ref{lem:gumbel-max} once $\delta=-\gamma_E$. The first displayed equation is the fixed-point equation of the soft Bellman operator $\cT_{\mathrm{soft}}$, which has a unique bounded fixed point by Lemma~\ref{lem:soft-contraction}. Thus both models yield the same $Q^\star$. Applying the common softmax equation gives the same $\pi^\star$, and the observed trajectory law depends only on $(\nu_0,P,\pi^\star)$.
\end{proof}

\begin{remark}[Nonzero T1EV means]
\label{rem:ddc-nonzero-mean}
If $\delta\neq-\gamma_E$, the DDC equation is~\eqref{eq:ddc-q-bellman-delta}. Relative to the mean-zero equation, its solution is shifted by the constant $\beta(\delta+\gamma_E)/(1-\beta)$, so the induced softmax policy is unchanged. The clean literal equality of the DDC and unit-entropy MaxEnt-IRL Bellman equations is obtained under the paper's normalization $\delta=-\gamma_E$.
\end{remark}

\begin{remark}[Two views, one model]
\label{rem:two-views}
The two formulations differ in how they explain stochastic action choices. DDC attributes them to unobserved T1EV utility shocks; MaxEnt-IRL attributes them to entropy-regularized stochastic choice. Under $\lambda=1$ and $\delta=-\gamma_E$, the choice-specific Bellman equation and the induced policy are identical, so the two formulations are statistically indistinguishable from offline observations of $(s,a,s')$.
\end{remark}

\subsection{The Sample Soft Bellman Operator and Double Sampling}
\label{sec:sample-soft-bellman}

Given $\cD = \{(s_i, a_i, s_i')\}_{i=1}^N$, the \emph{sample soft Bellman operator} is
\begin{equation*}
\hat{\cT}_{\mathrm{soft}} Q(s, a, s') \;:=\; r(s, a) + \beta\,\LSE(Q(s', \cdot)).
\end{equation*}
For each fixed $(s, a)$, $\EE_{s' \sim P(\cdot \mid s, a)}[\hat{\cT}_{\mathrm{soft}} Q(s, a, s')] = \cT_{\mathrm{soft}} Q(s, a)$: one sample is unbiased for the operator's value. Squaring breaks this. Define
\begin{align}
\cL_{\mathrm{BE}}(Q)(s, a) &\;:=\; \bigl(Q(s, a) - \cT_{\mathrm{soft}} Q(s, a)\bigr)^{\!2},\label{eq:LBE-def}\\
\cL_{\mathrm{TD}}(Q)(s, a, s') &\;:=\; \bigl(Q(s, a) - \hat\cT_{\mathrm{soft}} Q(s, a, s')\bigr)^{\!2},\label{eq:LTD-def}
\end{align}
and write $V_Q(s') := \LSE(Q(s', \cdot))$.

\begin{lemma}[Variance bias of the squared TD loss]
\label{lem:double-sampling}
For any $Q \in \cQ$ and any $(s, a)$,
\begin{equation}
\EE_{s' \sim P(\cdot \mid s, a)}\!\bigl[\cL_{\mathrm{TD}}(Q)(s, a, s')\bigr]
\;=\; \cL_{\mathrm{BE}}(Q)(s, a) \,+\, \beta^2\,\Var_{s' \sim P(\cdot \mid s, a)}\bigl(V_Q(s')\bigr).
\label{eq:double-sampling}
\end{equation}
\end{lemma}

\begin{proof}
Fix $(s, a)$ and let $X := \hat\cT_{\mathrm{soft}} Q(s, a, s') - Q(s, a)$, regarded as a random variable in $s'$. The variance decomposition $\EE[X^2] = (\EE X)^2 + \Var(X)$ gives
\begin{equation*}
\EE_{s'}[X^2] = \bigl(\cT_{\mathrm{soft}} Q(s, a) - Q(s, a)\bigr)^{\!2} + \Var_{s'}\bigl(\hat\cT_{\mathrm{soft}} Q(s, a, s')\bigr).
\end{equation*}
Since $r(s, a)$ is constant in $s'$, $\Var_{s'}(\hat\cT_{\mathrm{soft}} Q) = \Var_{s'}(r(s, a) + \beta V_Q(s')) = \beta^2 \Var_{s'}(V_Q(s'))$.
\end{proof}

The discrepancy $\beta^2 \Var_{s' \sim P(\cdot \mid s, a)}(V_Q(s'))$ is the \emph{double-sampling bias}. It is a $Q$-dependent population quantity that does not vanish with more data, so its gradient is nonzero at $Q^\star$. \emph{Minimizing the empirical squared TD residual is not the same as minimizing the population Bellman error.} Two independent next-state samples from the same $(s, a)$ would suffice to break the bias, but offline RL never gives us two next-states at the same $(s, a)$. Section~\ref{sec:erm-bias-correction} fixes this by deriving the bi-conjugate correction explicitly.

\bigskip

With the equivalence and the sampling subtlety in hand, we proceed to identification (Section~\ref{sec:identification}), the classical DDC methods (Section~\ref{sec:ddc-methods}), the IRL methods (Section~\ref{sec:irl}), and the unified ERM treatment (Section~\ref{sec:erm}).

\section{Identification: What Can Be Recovered from Behavior Alone?}
\label{sec:identification}

Before discussing how to estimate the reward, we have to confront a more basic question: \emph{is the reward identifiable from offline expert data at all?}

Without further structure, the answer is no. This is not a quirk of any particular estimator; it is a fact about the model. We make it precise, then describe the two normalizations, one from econometrics and one from machine learning, that resolve it, and assess what each delivers \citep{rust1994structural,magnac2002identifying,cao2021identifiability}.

\subsection{The Identification Failure}

Suppose two reward functions $r$ and $r'$ produce, via the soft Bellman equation, the same optimal policy $\pi^\star$ and the same trajectory distribution under $P$. Then the offline expert data $\cD$ has identical distribution under $(r, \pi^\star)$ and $(r', \pi^\star)$, and no estimator based on $\cD$ alone can possibly distinguish them. Worse, the indeterminacy is large: the soft-Bellman fixed point is invariant under \emph{state-only potential-based shaping} in the sense of \citet{ng1999policy}, and we can write down explicitly an entire family of rewards that all produce the same expert behavior. This non-identification was established in the DDC framework by \citet{magnac2002identifying} (their Proposition~1); we restate it in the entropy-regularized language for self-containedness.

\begin{lemma}[Non-identification from a single expert policy; {\citealp{magnac2002identifying}}]
\label{lem:nonid}
Fix transition $P$ and discount $\beta$. Let $r$ be any reward function, with $Q^\star$ and $\pi^\star$ the soft optimal pair (Theorem~\ref{thm:equivalence}). For any bounded measurable $\Phi : \cS \to \R$, define the shaped reward
\begin{equation}
\tilde r(s, a) \;:=\; r(s, a) + \Phi(s) - \beta\,\EE_{s' \sim P(\cdot \mid s, a)}[\Phi(s')].
\label{eq:shaping}
\end{equation}
Let $\tilde Q(s, a) := Q^\star(s, a) + \Phi(s)$. Then $\tilde Q$ is the unique soft-Bellman fixed point of $\tilde r$, and its induced softmax policy equals $\pi^\star$.
\end{lemma}

\begin{proof}
\emph{Bellman fixed-point check.} Apply $\cT_{\mathrm{soft}}^{\tilde r}$ to $\tilde Q$. By Lemma~\ref{lem:lse-shift} applied with $c = \Phi(s')$,
\begin{equation*}
\LSE(\tilde Q(s', \cdot)) \;=\; \LSE\!\bigl(Q^\star(s', \cdot) + \Phi(s')\bigr) \;=\; \LSE(Q^\star(s', \cdot)) + \Phi(s').
\end{equation*}
Therefore
\begin{align*}
\cT_{\mathrm{soft}}^{\tilde r}\tilde Q(s, a)
&= \tilde r(s, a) + \beta\, \EE_{s' \sim P(\cdot \mid s, a)}\!\bigl[\LSE(\tilde Q(s', \cdot))\bigr]\\
&= r(s, a) + \Phi(s) - \beta\,\EE_{s'}[\Phi(s')] + \beta\,\EE_{s'}\!\bigl[\LSE(Q^\star(s', \cdot)) + \Phi(s')\bigr]\\
&= r(s, a) + \Phi(s) + \beta\,\EE_{s'}[\LSE(Q^\star(s', \cdot))]\\
&= Q^\star(s, a) + \Phi(s) \;=\; \tilde Q(s, a),
\end{align*}
where the second-to-last equality uses the soft Bellman equation for $r$. So $\tilde Q$ is a fixed point of $\cT_{\mathrm{soft}}^{\tilde r}$; by Lemma~\ref{lem:soft-contraction} it is the unique such fixed point.

\emph{Policy invariance.} For each $s$, by Lemma~\ref{lem:lse-shift},
\begin{equation*}
\frac{\exp \tilde Q(s, a)}{\sum_{a'} \exp \tilde Q(s, a')} \;=\; \frac{e^{\Phi(s)} \exp Q^\star(s, a)}{e^{\Phi(s)} \sum_{a'} \exp Q^\star(s, a')} \;=\; \frac{\exp Q^\star(s, a)}{\sum_{a'} \exp Q^\star(s, a')} \;=\; \pi^\star(a \mid s).
\end{equation*}
\end{proof}

\citet{magnac2002identifying} proved that, in the DDC framework, this is essentially the \emph{only} non-identification: any two per-period utilities producing the same choice probabilities under the same $(P, \beta)$ differ by a state-only potential. \citet{cao2021identifiability} gave the corresponding result explicitly in the entropy-regularized MDP language. To make the result fully concrete:

\begin{theorem}[Reward identifiability up to shaping; \citealp{magnac2002identifying,cao2021identifiability}]
\label{thm:cao-id}
Fix $(P, \beta)$ and let $r_1, r_2$ be two bounded reward functions on $\cS \times \cA$. Let $\pi^\star_1, \pi^\star_2$ be the corresponding entropy-regularized optimal policies. Then $\pi^\star_1 = \pi^\star_2$ if and only if there exists a bounded measurable $\Phi : \cS \to \R$ such that
\begin{equation*}
r_1(s, a) - r_2(s, a) \;=\; \Phi(s) - \beta\,\EE_{s' \sim P(\cdot \mid s, a)}[\Phi(s')] \qquad \forall (s, a).
\end{equation*}
In particular, \textbf{(i)} \emph{state-action} rewards are identified only modulo this shaping family, regardless of how rich the data is; and \textbf{(ii)} \emph{state-only} rewards $r(s)$ are identified up to additive constant if and only if the only $\Phi$ for which $\Phi(s) - \beta\,\EE_{s' \sim P(\cdot \mid s, a)}[\Phi(s')]$ is a function of $s$ alone is the constant function, a condition on the transition kernel $P$.
\end{theorem}

Lemma~\ref{lem:nonid} establishes the ``if'' direction; the ``only if'' is the content of \citet{magnac2002identifying} and \citet{cao2021identifiability}. The structural takeaway is the following:

\begin{quote}
\emph{From expert behavior alone, the reward function is identifiable only up to a state-dependent potential. To identify it any more sharply, we need additional structure: either restrictions on the form of $r$, or normalizations that fix the potential.}
\end{quote}

The DDC and modern-IRL literatures take very different roads at exactly this fork. We examine both.

\subsection{The Anchor-Action Assumption (Magnac--Thesmar)}
\label{sec:anchor}

The classical DDC route is to declare the value of $r$ at one designated action per state \citep{magnac2002identifying}.

\begin{assumption}[Anchor action]
\label{ass:anchor}
For every $s \in \cS$ there exists a known, distinguished action $a_s \in \cA$ with
\begin{equation*}
r(s, a_s) \;=\; 0.
\end{equation*}
\end{assumption}

In the Rust bus-engine application, the anchor action is ``replace the engine''; its payoff (the cost of a new engine) is part of the model specification, not something to be estimated from data. The anchor action is allowed to depend on $s$.

\begin{theorem}[Identification under the anchor assumption; \citealp{magnac2002identifying}]
\label{thm:magnac-thesmar}
Under Assumption~\ref{ass:anchor}, $Q^\star(s, a)$ and $r(s, a)$ are identified for all $(s, a)$ from the triple $(\pi^\star, P, \beta)$.
\end{theorem}

\begin{proof}
We proceed in three steps.

\emph{Step 1: Hotz--Miller inversion gives $Q^\star$ differences.}
From the softmax~\eqref{eq:opt-policy}, $\pi^\star(a \mid s) = e^{Q^\star(s, a)}/\sum_{a'} e^{Q^\star(s, a')}$. Taking logs and differencing,
\begin{equation}
\log \pi^\star(a \mid s) - \log \pi^\star(a_s \mid s) \;=\; Q^\star(s, a) - Q^\star(s, a_s) \qquad \forall a \in \cA, s \in \cS.
\label{eq:Hotz-Miller-inversion}
\end{equation}
Hence $Q^\star(s, a)$ is identified for every $a$ once $Q^\star(s, a_s)$ is.

\emph{Step 2: A linear fixed-point equation for $V^\star$ alone.}
Taking logs of the softmax at the anchor action and rearranging,
\begin{equation}
Q^\star(s, a_s) \;=\; \LSE(Q^\star(s, \cdot)) + \log \pi^\star(a_s \mid s) \;=\; V^\star(s) + \log \pi^\star(a_s \mid s),
\label{eq:HM-Q-at-anchor}
\end{equation}
where $V^\star(s) := \LSE(Q^\star(s, \cdot))$. The soft Bellman equation~\eqref{eq:soft-bellman} at the anchor action, using $r(s, a_s) = 0$, gives
\begin{equation}
Q^\star(s, a_s) \;=\; \beta\,\EE_{s' \sim P(\cdot \mid s, a_s)}[V^\star(s')].
\label{eq:Q-anchor-bellman}
\end{equation}
Combining~\eqref{eq:HM-Q-at-anchor} and~\eqref{eq:Q-anchor-bellman},
\begin{equation}
V^\star(s) \;=\; \beta\,\EE_{s' \sim P(\cdot \mid s, a_s)}[V^\star(s')] \;-\; \log \pi^\star(a_s \mid s).
\label{eq:V-fixedpoint}
\end{equation}
This is a Bellman-like fixed-point equation for $V^\star$, with $-\log \pi^\star(a_s \mid \cdot)$ playing the role of a reward and the anchor-conditioned kernel $P(\cdot \mid \cdot, a_s)$ as the transition. The operator $\Psi V(s) := -\log \pi^\star(a_s \mid s) + \beta\,\EE_{s' \sim P(\cdot \mid s, a_s)}[V(s')]$ is a $\beta$-contraction on the space of bounded functions $\cS \to \R$ in sup-norm by the same Jensen argument as Lemma~\ref{lem:soft-contraction} (simpler, since no $\LSE$ is involved). The pseudo-reward $-\log \pi^\star(a_s \mid s)$ is itself bounded under the standing assumption $\inf_s \pi^\star(a_s \mid s) > 0$, which holds automatically when $\pi^\star$ is the softmax of a bounded $Q^\star$ on a finite action set (Remark~\ref{rem:support-auto} below). By the Banach fixed-point theorem, $V^\star$ is the unique solution of~\eqref{eq:V-fixedpoint} and is therefore identified from $(\pi^\star, P, \beta)$ alone, with no reference to $r$.

\emph{Step 3: Recover $Q^\star$ and $r$.}
From Step~2 we have $V^\star$. From~\eqref{eq:HM-Q-at-anchor} we obtain $Q^\star(s, a_s) = V^\star(s) + \log \pi^\star(a_s \mid s)$. From~\eqref{eq:Hotz-Miller-inversion} we then obtain $Q^\star(s, a)$ for every $a$. Finally, from the soft Bellman equation,
\begin{equation}
r(s, a) \;=\; Q^\star(s, a) - \beta\,\EE_{s' \sim P(\cdot \mid s, a)}[V^\star(s')].
\label{eq:r-recovery}
\end{equation}
\end{proof}

A small but important point: the anchor-action assumption is a \emph{normalization}, not a restriction. It fixes the location of the reward without restricting which preferences over $(s, a)$ pairs the model can represent. It is the standard identification device in structural econometrics for DDC \citep{rust1994structural,magnac2002identifying,arcidiacono2011conditional,arcidiacono2011practical}.

\subsection{The Adversarial-IRL Route and Its Limits}
\label{sec:airl-id}

The machine-learning tradition takes a different fork at the identification problem of Lemma~\ref{lem:nonid}. The DDC route (Section~\ref{sec:anchor}) fixes the shaping potential by fiat: declare one action's reward known at every state. Adversarial IRL \citep{fu2017learning} instead tries to \emph{learn} the split between reward and potential, the state-dependent term whose Bellman difference can change the numerical reward without changing behavior, by building the potential-based shaping invariance of \citet{ng1999policy} directly into the structure of a discriminator. This subsection asks how much that buys.

The guiding distinction is simple. From transitions $(s,a,s')$, expert behavior directly reveals which actions look better \emph{relative to other actions at the same state}; it does not directly reveal how much of that attractiveness comes from immediate reward rather than continuation value. AIRL is an attempt to learn this split: one part should be the transferable reward, and the other part should be the shaping/value correction. The question is whether the split is actually identified.

\subsubsection*{The AIRL language and objective.}

\citet{fu2017learning} build on the MaxEnt-IRL framework of \citet{finn2016connection} and the generative-adversarial framework of \citet{goodfellow2020generative,ho2016generative}. AIRL introduces a binary classifier $D$, called the \emph{discriminator}. Given a transition $x=(s,a,s')$, $D(x)$ is trained to be large on expert transitions and small on transitions generated by the current policy $\pi$. Before writing $D$ itself, AIRL specifies the reward-like score that will enter the classifier. That score is constrained to the potential-shaped form
\begin{equation}
f_{g,h}(s, a, s') \;:=\; g(s) \;+\; \beta\, h(s') \;-\; h(s),
\label{eq:airl-parametrization}
\end{equation}
where $g$ is intended as the transferable reward and $h$ as the shaping potential, which is equivalently the soft value term that should telescope away. In words, AIRL writes the expert-like score as
\[
\text{reward part} + \text{future potential} - \text{current potential}.
\]
The hope is that the reward part $g$ survives transfer, while the potential terms only correct for where the transition starts and ends. We write the state-only reward-head version here because it is the version under which AIRL's disentanglement claim can be true; if $g$ is allowed to depend on $(s,a)$, the same ambiguity returns immediately. In the terminology of \citet{cao2021identifiability}, ``state-only'' is an action-independent reward restriction. It is a real identifying restriction: recovering such a reward also needs decomposability conditions on how actions move between states, not only the adversarial objective.

This potential-shaped quantity is only a score. To turn such a score into a classifier, the adversarial-IRL derivation starts from the standard GAN/Bayes-classifier density-ratio identity. For a binary classifier trained to classify expert samples against current-policy samples, the Bayes classifier with equal class priors has
\begin{equation}
D^\star(x)=\frac{p_E(x)}{p_E(x)+p_\pi(x)},
\qquad
\log\frac{D^\star(x)}{1-D^\star(x)}
= \log p_E(x)-\log p_\pi(x).
\label{eq:bayes-log-density-ratio}
\end{equation}
Thus the classifier logit is a log density ratio. The next step is to specify the two trajectory densities whose ratio will be used. Write a finite trajectory as $\tau=(s_0,a_0,s_1,a_1,\ldots,s_{H-1},a_{H-1},s_H)$. The density-ratio motivation used in GAN-guided cost learning and AIRL starts from an energy model over trajectories: keep the environment part of the trajectory probability, $\nu_0(s_0)\prod_t P(s_{t+1}\mid s_t,a_t)$, and tilt it toward trajectories whose accumulated learned score $\sum_t f(s_t,a_t,s_{t+1})$ is large. Exponential tilting gives the expert-side density
\begin{equation*}
p_f(\tau)
\;\propto\;
\nu_0(s_0)\prod_{t=0}^{H-1}P(s_{t+1}\mid s_t,a_t)
\exp\!\Bigl(\sum_{t=0}^{H-1} f(s_t,a_t,s_{t+1})\Bigr),
\end{equation*}
whereas the current policy generates trajectories according to
\begin{equation*}
p_\pi(\tau)
\;=\;
\nu_0(s_0)\prod_{t=0}^{H-1}\pi(a_t\mid s_t)P(s_{t+1}\mid s_t,a_t).
\end{equation*}
Taking the log-ratio cancels the common initial-state and transition factors:
\begin{equation*}
\log p_f(\tau)-\log p_\pi(\tau)
\;=\;
\sum_{t=0}^{H-1}\bigl[f(s_t,a_t,s_{t+1})-\log\pi(a_t\mid s_t)\bigr]
\;+\; \text{constant}.
\end{equation*}
By the Bayes log-density-ratio identity~\eqref{eq:bayes-log-density-ratio}, the corresponding trajectory classifier has log-odds
\begin{equation}
\log \frac{D_f^{\pi}(\tau)}{1-D_f^{\pi}(\tau)}
\;=\;
\sum_{t=0}^{H-1}\bigl[f(s_t,a_t,s_{t+1})-\log\pi(a_t\mid s_t)\bigr]
\;+\; C_f,
\label{eq:airl-trajectory-log-odds}
\end{equation}
where \(C_f\) absorbs the trajectory-energy normalizer and any class-prior constant. This trajectory classifier is the GAN-GCL starting point.

\paragraph{From trajectory odds to the local AIRL discriminator.}
AIRL \citep{fu2017learning} now makes a one-step localization of the trajectory-level odds relation. Since the trajectory logit in~\eqref{eq:airl-trajectory-log-odds} is a sum of per-transition terms, the natural local contribution of a transition $(s,a,s')$ is
\[
 f(s,a,s')-\log\pi(a\mid s).
\]
AIRL therefore defines the transition-level discriminator by the odds equation
\begin{equation}
\log \frac{D_f^{\pi}(s,a,s')}{1-D_f^{\pi}(s,a,s')}
\;=\;
f(s,a,s')-\log\pi(a\mid s).
\label{eq:airl-log-odds}
\end{equation}
Equivalently,
\begin{equation}
D_f^{\pi}(s,a,s')
\;=\;
\frac{\exp f(s,a,s')}{\exp f(s,a,s')+\pi(a\mid s)}.
\label{eq:airl-generic-discriminator}
\end{equation}
How can we justify this localization? The clean way is not to view it as literally marginalizing the trajectory Bayes classifier down to one transition. A true marginal classifier for a single transition would have to account for the full distribution of times and histories that can produce the same $(s,a,s')$. The justification is more local: because the trajectory logit in~\eqref{eq:airl-trajectory-log-odds} is additive, the term $f(s,a,s')-\log\pi(a\mid s)$ is the one-transition contribution suggested by that logit. The lemma below records why this is the right surrogate for AIRL: it gives the local discriminator gradient used in the adversarial update and the entropy-regularized reward used in the policy update.

\begin{lemma}[Trajectory-to-transition localization, informal]
\label{lem:airl-trajectory-to-transition}
Fix a policy $\pi$ and a transition score $f$. The transition discriminator~\eqref{eq:airl-log-odds} is the one-step local surrogate associated with the trajectory log-odds~\eqref{eq:airl-trajectory-log-odds}. Its log-odds reward is
\begin{equation*}
\log D_f^{\pi}(s,a,s')-\log(1-D_f^{\pi}(s,a,s'))
=
f(s,a,s')-\log\pi(a\mid s),
\end{equation*}
so using this log-odds as the policy reward makes the policy maximize the expected sum of $f$ plus the Shannon entropy of $\pi$. At the population level, the corresponding discriminator gradient has the same form as the maximum-causal-entropy IRL gradient after the intractable trajectory marginal is replaced by a local policy-based energy measure. The precise finite-horizon statement and proof are deferred to Proposition~\ref{prop:airl-local-identities} in Appendix~\ref{app:airl-local-identities}.
\end{lemma}

Thus the disappearance of the trajectory-level $\sum_t$ is the answer to a localization question, not a new cancellation in the likelihood ratio: AIRL keeps the one-transition contribution because it has the right discriminator-gradient and policy-objective identities, not because the full trajectory Bayes classifier literally collapses to one step.

This is also the right place to keep one interpretive caveat in view. The trajectory-energy model above and its corresponding localization is not the same object as the dynamic soft-control model in Section~\ref{sec:irl-setup}. It is \emph{myopic} in the specific sense emphasized by \citet{cao2021identifiability}: the trajectory likelihood is tilted by accumulated one-step scores \(f(s_t,a_t,s_{t+1})\), whereas a dynamically optimal policy evaluates an action through immediate reward plus the continuation value induced by the transition kernel. In the dynamic model,
\[
\log \pi^\star(a\mid s)
=
Q^\star(s,a)-V^\star(s)
=
r(s,a)+\beta\,\EE[V^\star(s')\mid s,a]-V^\star(s).
\]
So the local AIRL score should first be read as a policy-rationalizing advantage-like object. It becomes a reward only after a separate decomposition argument explains which part is immediate payoff and which part is value shaping.

Specializing the generic local discriminator~\eqref{eq:airl-generic-discriminator} to the structured AIRL score $f=f_{g,h}$ from~\eqref{eq:airl-parametrization} gives
\begin{equation}
D_{g,h}^{\pi}(s, a, s') \;=\; \frac{\exp f_{g,h}(s, a, s')}{\exp f_{g,h}(s, a, s') + \pi(a \mid s)}.
\label{eq:airl-discriminator}
\end{equation}
The $\pi(a\mid s)$ term has an intuitive role: a generated transition may contain action $a$ simply because the current policy already takes $a$ often at state $s$. AIRL therefore asks whether the expert-like score $\exp f_{g,h}(s,a,s')$ is large \emph{relative to} the current policy's own probability of producing that action. This is why the discriminator is policy-dependent, written $D_{g,h}^{\pi}$.

In the trained model $g = g_\theta$ and $h = h_\phi$ are parametrized, with $f_{\theta,\phi}:=f_{g_\theta,h_\phi}$ and $D_{\theta,\phi}^{\pi} := D_{g_\theta,h_\phi}^{\pi}$. We write $f_{g,h}$ when reasoning about identification at the level of functions, and $f_{\theta,\phi}$ when reasoning about the optimization. Also, in the adversarial objectives below, the notation $(s,a,s')\sim d^\pi$ means $(s,a)\sim d^\pi$ followed by $s'\sim P(\cdot\mid s,a)$; equivalently it is the transition-augmented discounted occupancy.

The two training updates now follow from the two roles of this classifier. 
\begin{itemize}[leftmargin=*]
    \item First, fix the current policy $\pi$. The policy-generated transitions are the negative samples for the discriminator. If $y=1$ denotes an expert transition and $y=0$ a generated transition, then the discriminator log-likelihood is
\begin{equation*}
y\log D_{\theta,\phi}^{\pi}(s,a,s')+(1-y)\log\!\bigl(1-D_{\theta,\phi}^{\pi}(s,a,s')\bigr).
\end{equation*}
Averaging this log-likelihood over expert samples from $d^{\pi^\star}$ and generated samples from $d^\pi$, and dropping the irrelevant class-prior factor, gives the discriminator objective: increase $\log D_{\theta,\phi}^{\pi}$ on expert transitions and increase $\log(1-D_{\theta,\phi}^{\pi})$ on policy transitions.
\item Second, fix $(\theta,\phi)$ and update the policy. AIRL gives the policy the same discriminator's log-odds as its reward signal. Substituting $g=g_\theta$ and $h=h_\phi$ into the log-odds identity~\eqref{eq:airl-log-odds} gives
\begin{equation*}
\log D_{\theta,\phi}^{\pi}(s,a,s')
- \log\!\bigl(1-D_{\theta,\phi}^{\pi}(s,a,s')\bigr)
= f_{\theta,\phi}(s,a,s')-\log \pi(a\mid s).
\end{equation*}
Thus averaging the log-odds over transitions from $d^\pi$ gives an entropy-regularized policy objective: the policy optimizer sees the learned score $f_{\theta,\phi}$ together with the entropy term induced by $-\log\pi(a\mid s)$, rather than a raw logistic score.
Operationally, the discriminator says which transitions look expert-like relative to the current policy, and the policy update tries to make those high-log-odds transitions more common.
\end{itemize}

Putting these two derived objectives together, AIRL alternates:
\begin{align}
\text{Discriminator } (\max_{\theta, \phi}):\;\;
&\EE_{(s, a, s') \sim d^{\pi^\star}}[\log D_{\theta,\phi}^{\pi}(s, a, s')] + \EE_{(s, a, s') \sim d^\pi}[\log(1 - D_{\theta,\phi}^{\pi}(s, a, s'))],
\label{eq:airl-disc-update}\\
\text{Policy } (\max_\pi):\;\;
&\EE_{(s, a, s') \sim d^\pi}\!\bigl[\log D_{\theta,\phi}^{\pi} - \log(1 - D_{\theta,\phi}^{\pi})\bigr] \;=\; \EE_{d^\pi}\!\bigl[f_{\theta,\phi} - \log \pi(a \mid s)\bigr],
\nonumber
\end{align}
so the policy maximizes entropy-regularized return under the implicit reward $f_{\theta,\phi}$. The modeling bet is that~\eqref{eq:airl-parametrization} splits the soft advantage cleanly into a reward part and a shaping part. The bet can pay off only if that split is \emph{unique}; the rest of this subsection is about when it is.

\subsubsection*{Identification in AIRL}

We are finally ready to discuss the identification logic of AIRL. We discuss that:  
\begin{itemize}[leftmargin=2em]
\item When the adversarial game actually reaches its saddle, under deterministic transitions, the AIRL discriminator identifies the \emph{soft advantage} $A^\star(s, a) = \log \pi^\star(a \mid s)$. This is exactly the expert's conditional choice probabilities (the same information Hotz--Miller inversion provides) and nothing more.
\item Converting that advantage into a \emph{reward} requires a separate decomposition argument from \citet{fu2017learning}. We show that it works, but under four restrictions: a state-only reward, deterministic dynamics, a connectivity-and-coverage condition on the transition graph, and an attained exact saddle.
\item Each of the four is load-bearing. Drop any one and the shaping ambiguity of Lemma~\ref{lem:nonid} returns. Stochastic transitions, in particular, reintroduce it through a transition-kernel completeness condition that AIRL cannot enforce.
\end{itemize}

\paragraph{The saddle identifies the advantage, not the reward.}
The next statement is conditional and population-level. It assumes (a)~the policy player has reached the expert soft-optimal policy, so $d^\pi = d^{\pi^\star}$ on the support under discussion, and (b)~the discriminator class can realize its Bayes optimum there. These assumptions say that expert and generated transitions are truly indistinguishable to the classifier, and that the classifier class is rich enough to express the corresponding optimum. If either fails, the conclusions below do not apply: the logit would describe the current policy rather than the expert, or the AIRL class would not contain the Bayes-optimal discriminator.

Under (a)--(b), the best discriminator cannot tell the two sources apart, so $D^\star \equiv \tfrac12$ on the matched support. Since the policy has matched the expert, the $\pi$ appearing in~\eqref{eq:airl-discriminator} is now $\pi^\star$. Hence $D_{\theta,\phi}^{\pi^\star} = \tfrac12$ exactly when $\exp f_{\theta,\phi} = \pi^\star(a \mid s)$, which implies
\begin{equation}
f_{\theta,\phi}(s, a, s') \;=\; \log \pi^\star(a \mid s) \;=\; Q^\star(s, a) - V^\star(s) \;=:\; A^\star(s, a),
\qquad V^\star(s) := \LSE(Q^\star(s, \cdot)).
\label{eq:airl-saddle}
\end{equation}
So an attained, realizable saddle pins down the \emph{soft advantage} $A^\star$. Two consequences are worth stating plainly.

\emph{This is not new information.} In population, $A^\star(s, a) = \log \pi^\star(a \mid s)$ is just the expert's log conditional choice probability. Hotz--Miller inversion~\eqref{eq:Hotz-Miller-inversion} already extracts the same object through the log-probability differences $\log \pi^\star(a \mid s) - \log \pi^\star(a_s \mid s)$. This leads to the following Remark \ref{rmk:trivial}.

\begin{remark}\label{rmk:trivial}
    The preceding discussion also clarifies what AIRL gains by adding the anchor actions, e.g., 
an exit action or outside-option normalization. Under deterministic transitions (and only under deterministic transitions, as we will discuss), AIRL saddle supplies \(A^\star(s,a)=\log\pi^\star(a\mid s)\), which is exactly what Hotz-Miller achieves with MLE. Therefore, if an anchor action
\(a_s\) with known payoff is then imposed, the recovery of \(V^\star\),
\(Q^\star\), and \(r(s,a)\) follows from the same anchor-action logic in \cite{magnac2002identifying}.
\end{remark}

This is the key interpretive point. The advantage says how attractive action $a$ is relative to the baseline value of state $s$. It does not say whether that attractiveness came from high immediate reward, a favorable next-state distribution, or a potential-shaping term. Behavior may tell us that action $a$ is preferred to action $b$ at state $s$, but it does not by itself say which part is immediate payoff and which part is future continuation value. So the equality $f_{\theta,\phi}=A^\star$ is not yet reward recovery; it is only the object that reward recovery must decompose. 

\emph{Exact equality may be unreachable in stochastic environments.} The target $A^\star(s, a)$ does not depend on the realized next state $s'$, whereas the AIRL class $g(s) + \beta h(s') - h(s)$ generally does. So in a stochastic environment the pointwise equality~\eqref{eq:airl-saddle} can fail within the parametrized class, before any reward question is even posed. We return to this in the discussion of restriction~(iii) below; until then, transitions are deterministic.

\paragraph{From advantage to reward: decomposable deterministic dynamics.}
Assume now $s' = T(s, a)$ deterministic. Reward recovery still does not follow for free. If the AIRL reward head is allowed to be a state-action function, then for \emph{any} bounded $h : \cS \to \R$ the choice
\begin{equation*}
g_h(s, a) \;:=\; \log \pi^\star(a \mid s) - \beta\, h(T(s, a)) + h(s)
\end{equation*}
satisfies $g_h(s, a) + \beta h(T(s, a)) - h(s) = \log \pi^\star(a \mid s)$ exactly, so infinitely many decompositions of the same saddle logit coexist. Thus writing $f=g+\beta h'-h$ does not, by itself, tell us which piece is the reward. AIRL needs extra structure that makes the split unique.

The relevant property is not that next states identify current states; it is almost the opposite; enough overlap in one-step reachability that a next-state-only term cannot hide inside a sum of a current-state term and a next-state term. Define the one-step successor and predecessor sets
\begin{equation*}
\cS_{+} := \{x \in \cS : \exists\,(s, a) \text{ with } P(x \mid s, a) > 0\},
\qquad
\cS_{-} := \{s \in \cS : \exists\,(a, x) \text{ with } P(x \mid s, a) > 0\}.
\end{equation*}
In a standard MDP $\cS_{-} = \cS$, but $\cS_{+} = \cS$ only if every state has a one-step in-neighbor. Build an undirected graph on $\cS_{+}$, placing an edge between successors $x$ and $y$ whenever they share a predecessor (some $s$ and actions $a, b$ with $P(x \mid s, a) > 0$ and $P(y \mid s, b) > 0$) and close the relation transitively. The transition kernel is \emph{decomposable on the relevant state space} (the Fu--Luo--Levine condition) if this successor graph is connected and, for any claim of equality on all of $\cS$, additionally $\cS_{+} = \cS_{-} = \cS$.

The useful consequence is the following separation lemma. The domain qualification matters: the equation only sees $v$ at states that appear as successors, and only sees $u$ at states that appear as predecessors.

\begin{lemma}[Separation under decomposability]
\label{lem:fu-decomp-separation}
Suppose the successor graph on $\cS_{+}$ is connected. Let $u, v : \cS \to \R$ be bounded functions with
\begin{equation}
u(s) + v(s') = 0
\label{eq:decomp-sum-zero}
\end{equation}
for every supported transition $(s, a, s')$, i.e.\ whenever $P(s' \mid s, a) > 0$. Then $v$ is constant on $\cS_{+}$ and $u$ is constant on $\cS_{-}$. If moreover $\cS_{+} = \cS_{-} = \cS$, both functions are constant on all of $\cS$.
\end{lemma}

\begin{proof}
Fix a predecessor state $s \in \cS_{-}$. If $x$ and $y$ are both one-step successors of $s$ under possibly different actions, then~\eqref{eq:decomp-sum-zero} gives $u(s) + v(x) = 0$ and $u(s) + v(y) = 0$, hence $v(x) = v(y)$: $v$ is constant along every graph edge. Connectedness of the successor graph propagates the equality, so $v$ is constant on $\cS_{+}$. Substituting that constant back into~\eqref{eq:decomp-sum-zero} forces $u(s)$ to equal the negative of that constant for every $s \in \cS_{-}$.
\end{proof}

\begin{proposition}[AIRL disentanglement under deterministic decomposable dynamics]
\label{prop:airl-deterministic-decomp}
Assume the following.
\begin{enumerate}[leftmargin=2em,label=(\roman*)]
\item The ground-truth reward is state-only: $r(s, a) = r(s)$.
\item AIRL uses a state-only reward head and a state-only shaping head:
$f_{g,h}(s, a, s') = g(s) + \beta h(s') - h(s)$.
\item The transition is deterministic, so $s' = T(s, a)$.
\item The supported transitions cover and decompose the relevant state space: $\cS_{+} = \cS_{-} = \cS$, and the one-step successor-link graph on $\cS_{+}$ is connected.
\item The exact AIRL saddle is reached: on every supported transition,
$f_{g,h}(s, a, T(s, a)) = A^\star(s, a)$.
\end{enumerate}
Then there exists a scalar $c \in \R$ such that
\begin{equation*}
h(s) = V^\star(s) + c, \qquad g(s) = r(s) + (1 - \beta)c \qquad \text{for all } s.
\end{equation*}
In particular, AIRL's reward head identifies the true state-only reward up to an additive constant.
\end{proposition}

\begin{proof}
With a state-only reward and deterministic dynamics, the soft Bellman equation reads $Q^\star(s, a) = r(s) + \beta V^\star(T(s, a))$, so the soft advantage is
\begin{equation*}
A^\star(s, a) = Q^\star(s, a) - V^\star(s) = r(s) + \beta V^\star(T(s, a)) - V^\star(s).
\end{equation*}
Assumption~(v) then gives, on every supported transition,
\begin{equation*}
g(s) + \beta h(T(s, a)) - h(s) = r(s) + \beta V^\star(T(s, a)) - V^\star(s).
\end{equation*}
Write $\Delta := h - V^\star$ and $\delta := g - r$. Substituting $h = V^\star + \Delta$ and $g = r + \delta$ and cancelling the common $r, V^\star$ terms yields
\begin{equation}
\bigl(\delta(s) - \Delta(s)\bigr) + \beta\, \Delta(T(s, a)) = 0 \qquad \forall (s, a).
\label{eq:airl-decomp-proof}
\end{equation}
This is exactly the form handled by Lemma~\ref{lem:fu-decomp-separation}: set $u := \delta - \Delta$ and $v := \beta \Delta$. By the coverage and connectedness assumptions in~(iv), the lemma makes $v$, hence $\Delta$, constant on all of $\cS$; write $\Delta \equiv c$. Then~\eqref{eq:airl-decomp-proof} reads $\delta(s) - c + \beta c = 0$, so $\delta \equiv (1 - \beta)c$. Hence $h = V^\star + c$ and $g = r + (1 - \beta)c$.
\end{proof}

\paragraph{All four restrictions are load-bearing.}
Proposition~\ref{prop:airl-deterministic-decomp} is sharp: each substantive restriction removes a different obstruction. Here the state-only restriction combines assumptions~(i)--(ii), so the four restrictions are state-only structure, deterministic dynamics, decomposable coverage, and an attained exact saddle.

\emph{(i)--(ii) State-only rewards.} If the reward head may depend on $(s, a)$, the shaping family of Lemma~\ref{lem:nonid} comes straight back. For any bounded potential $\Phi : \cS \to \R$, set $h = V^\star + \Phi$ and
\begin{equation}
g_\Phi(s, a) := r(s, a) + \Phi(s) - \beta\, \EE_{s' \sim P(\cdot \mid s, a)}[\Phi(s')].
\label{eq:airl-state-action-shaping}
\end{equation}
Then $g_\Phi(s, a) + \beta\, \EE[h(s') \mid s, a] - h(s) = A^\star(s, a)$ identically. In deterministic dynamics this expectation is just evaluation at $T(s,a)$; in stochastic dynamics it is the population version of the same ambiguity. Either way, every potential-shaped reward produces the same soft advantage, so AIRL leaves state-action rewards unidentified beyond the baseline shaping equivalence.

\emph{(iii) Deterministic dynamics.} Determinism matters because then the next state $T(s,a)$ is fixed once $(s,a)$ is fixed, so the AIRL term $h(T(s,a))$ can line up pointwise with the Bellman equation. The proof used exactly this identity: $A^\star(s, a) = r(s) + \beta V^\star(T(s, a)) - V^\star(s)$. Under stochastic transitions, the Bellman identity carries a conditional expectation, while AIRL's realized term $h(s')$ depends on one sampled next state. The most one can ask is the population moment
\begin{equation}
\EE_{s' \sim P(\cdot \mid s, a)}[f_{g,h}(s, a, s')] = A^\star(s, a).
\label{eq:airl-expectation-condition}
\end{equation}
Even granting~\eqref{eq:airl-expectation-condition}, reward recovery is not automatic. With a state-only reward and state-only heads, the soft Bellman equation turns~\eqref{eq:airl-expectation-condition} into $g(s) + \beta \EE[h(s') \mid s, a] - h(s) = r(s) + \beta \EE[V^\star(s') \mid s, a] - V^\star(s)$; with $\Delta := h - V^\star$ and $\delta := g - r$ this rearranges to
\begin{equation}
\delta(s) = \Delta(s) - \beta\, \EE[\Delta(s') \mid s, a] \qquad \forall a.
\label{eq:airl-stochastic-delta}
\end{equation}
The left-hand side is independent of $a$, so $\EE[\Delta(s') \mid s, a]$ must be independent of $a$ as well. This forces $\Delta$ to be locally, and then globally, constant only if the transition family is rich enough that the condition ``all action-conditioned expectations of $\Delta$ agree'' implies $\Delta$ is constant: a completeness condition on $P$. Mere dependence of $P(\cdot \mid s, a)$ on $a$ is not enough, since two next-state distributions can share the expectation of a non-constant $\Delta$. Under the completeness condition one recovers $h = V^\star + c$ and $g = r + (1 - \beta)c$ as before; without it,~\eqref{eq:airl-stochastic-delta} leaves residual shaping ambiguity.

This distinction clarifies the role of transition-aware AIRL variants. \citet{zhan2024model} tries to address the stochastic-realizability obstruction: replacing AIRL's realized-next-state term \(h(s')\) by a model-based approximation to \(\EE[h(s')\mid s,a]\), equivalently a transition-aware shaping term such as \(R(s,a)+\beta\EE_{\hat T}[\varphi(s')\mid s,a]-\varphi(s)\), makes the logit compatible with the stochastic Bellman identity up to learned-transition error. That repair does not by itself supply the completeness or normalization condition needed for identification. After the expectation is inserted, the remaining equation is still~\eqref{eq:airl-stochastic-delta}, so reward recovery still requires that the only potential differences \(\Delta\) whose action-conditioned transition expectations agree are constants.

\emph{(iv) Connectivity and coverage.} If the successor graph splits into several components, the separation argument forces $\Delta$ constant only \emph{within} each component, and different per-component constants give genuinely different decompositions consistent with the same saddle. States that never appear as one-step successors leave $h$ unconstrained there; states that never appear as predecessors leave $g$ unconstrained. Deterministic systems with little branching, or with disconnected recurrent classes, routinely fail the condition.

\emph{(v) Exact saddle.} If the policy player has not matched the expert occupancy, or if the discriminator class cannot realize its Bayes optimum, the equality $f_{g,h}=A^\star$ need not hold on the expert support. The learned logit then reflects the current policy/discriminator game rather than the expert's soft advantage, so there is no exact advantage identity to decompose into reward plus shaping terms.

\paragraph{Relation to anchor actions.}
The bottom line is that AIRL is elegant because it builds potential shaping into the discriminator, but that removes the baseline identification problem only under assumptions such as state-dependent rewards and deterministic transitions. At the perfect saddle, it recovers the expert's soft advantage; reward recovery requires additional structure that makes the advantage decomposition unique.

To summarize:
\begin{itemize}[leftmargin=*]
    \item AIRL recovers Hotz-Miller conversion under deterministic transitions. In other words, under deterministic transitions, making the anchor action assumption additionally yields exact reward recovery \citep{magnac2002identifying}. However, under stochastic transitions, the Hotz-Miller conversation is not recovered because it stems from the trajectory energy model (which is fundamentally myopic under stochastic transitions \citep{cao2021identifiability}). 
    \item AIRL can avoid the anchor-action assumption, following \citet{fu2017learning}, by making extra assumptions such as deterministic transitions and state-only rewards. 
\end{itemize}

 For the remainder of these notes, we go back to the anchor action assumption Assumption~\ref{ass:anchor}: it identifies state-action rewards directly, on the expert support, without requiring state-only rewards, deterministic dynamics, or a decomposable transition graph.

\section{Classical DDC Methods}
\label{sec:ddc-methods}

We now survey the classical estimation routes that grew out of the Magnac--Thesmar identification: Rust's nested fixed-point algorithm (Section~\ref{sec:rust}), the conditional-choice-probability approach of Hotz--Miller (Section~\ref{sec:hotz-miller}), and the Adusumilli--Eckardt temporal-difference route (Section~\ref{sec:td-ddc}), which itself contains two methods: linear semi-gradient TD and approximate value iteration. Each route gets one of the following ingredients right, but pays a price elsewhere:
\begin{itemize}[leftmargin=2em]
\item \emph{Scalability} in the dimension of $\cS$,
\item \emph{Transition-kernel freedom} (no separate estimate of $P$ required),
\item \emph{Stability} (no deadly-triad-style divergence).
\end{itemize}
Section~\ref{sec:erm} is then about how we can address this issue.

The starting point is common to all. By Theorem~\ref{thm:magnac-thesmar}, $\pi^\star$ pins down $Q^\star$ once an anchor action is fixed. Operationally, this requires solving two coupled conditions (the softmax MLE condition and the soft Bellman equation at the anchor action) that we record now.

\subsection{The Two Pieces: NLL and Bellman}
\label{sec:two-pieces}

\begin{lemma}[Softmax MLE optimality]
\label{lem:nll}
Let $\rho \in \Delta(\cS \times \cA)$ have state marginal $\rho_S$ and conditional $\rho(\cdot \mid s) = \pi^\star(\cdot \mid s)$ for $\rho_S$-a.e.\ $s$. Then
\begin{equation*}
\argmin_{Q \in \cQ}\, \EE_{(s,a) \sim \rho}\!\bigl[-\log \hat p_Q(a \mid s)\bigr]
\;=\; \{Q \in \cQ : \hat p_Q(\cdot \mid s) = \pi^\star(\cdot \mid s)\;\rho_S\text{-a.e.}\},
\end{equation*}
where $\hat p_Q(a \mid s) := e^{Q(s, a)} / \sum_{a'} e^{Q(s, a')}$.
\end{lemma}

\begin{proof}
The objective decomposes as
\begin{equation*}
\EE_{(s,a) \sim \rho}\!\bigl[-\log \hat p_Q(a \mid s)\bigr]
\;=\; \EE_{s \sim \rho_S}\!\bigl[\KL(\pi^\star(\cdot \mid s)\,\|\,\hat p_Q(\cdot \mid s))\bigr] + \EE_{s \sim \rho_S}\!\bigl[\Ent(\pi^\star(\cdot \mid s))\bigr].
\end{equation*}
The second term is independent of $Q$. The first is non-negative and zero iff $\hat p_Q(\cdot \mid s) = \pi^\star(\cdot \mid s)$ for $\rho_S$-a.e.\ $s$ (Gibbs' inequality).
\end{proof}

Notice what the NLL gives us and what it does not. The softmax $\hat p_Q$ is invariant under state-only shifts $Q(s, a) \mapsto Q(s, a) + c(s)$ (Lemma~\ref{lem:lse-shift}). So fitting $\pi^\star$ alone pins down $Q$ only up to such shifts, exactly the shaping ambiguity of Lemma~\ref{lem:nonid}. \emph{The Bellman condition is what locks down those shifts.} The classical methods differ in how they enforce it.

\subsection{Rust's Nested Fixed-Point Algorithm (NFXP)}
\label{sec:rust}

\citet{rust1987optimal} proposed to parametrize the \emph{reward} $r_\theta$ in a low-dimensional class and treat the Bellman equation as an inner-loop constraint:
\begin{align*}
\text{Outer (over $\theta$):}\quad &\theta_{k+1} = \argmin_\theta\;\widehat\EE_{\cD}\bigl[-\log \hat p_{Q_\theta}(a \mid s)\bigr],\\
\text{Inner (at each $\theta$):}\quad &Q_\theta := \text{(unique fixed point of } \cT^{r_\theta}_{\mathrm{soft}})\text{ computed by soft value iteration.}
\end{align*}
Convergence of the inner loop is guaranteed by Lemma~\ref{lem:soft-contraction}; consistency of the outer MLE is the classical result of \citet{rust1987optimal}. \emph{This is the right thing to do, statistically; the problem is doing it.}

\paragraph{The computational obstruction.}
The inner soft value iteration is over $\cS \times \cA$. For each iteration, the LSE
\begin{equation*}
\LSE(Q(s', \cdot)) \;=\; \log \sum_{a' \in \cA} \exp(Q(s', a'))
\end{equation*}
must be evaluated at \emph{every} $s'$ the outer expectation might query. Discretizing $\cS$ to $|\cS|$ points and running value iteration costs $\Theta(|\cS| \cdot |\cA|)$ per iteration to convergence in $\Theta(\log(1/\eps)/(1-\beta))$ iterations, and the inner loop sits inside the outer loop over $\theta$. When the underlying state has intrinsic dimension $d$, fixed-grid discretization to resolution $h$ produces $|\cS| = \Theta(h^{-d})$, so the inner cost scales as $h^{-d}$: the \emph{curse of dimensionality of dynamic programming}. Workarounds (sieve approximations \citep{kristensen2021solving,arcidiacono2013approximating}, recursive partitioning \citep{barzegary2022recursive}, state aggregation \citep{geng2023data}, neural-network approximations \citep{norets2012estimation}) push the boundary but do not change the scaling.

\paragraph{The transition kernel is also needed.}
There is a second problem hidden inside the first. The inner soft value iteration requires the expectation $\EE_{s' \sim P(\cdot \mid s, a)}[\LSE(Q(s', \cdot))]$, which presumes either (i) closed-form access to $P$ (rare outside textbook problems), or (ii) a separately estimated $\hat P$. In option (ii), error in $\hat P$ propagates through the fixed-point computation; \citet{rust1994structural} requires $P$ to be known a priori or consistently estimable in a first stage, a non-trivial demand in high dimensions. A constrained-optimization reformulation (MPEC) that avoids the nested loop in low dimensions is given by \citet{su2012constrained}: treat both $\theta$ and the value function as optimization variables and minimize the likelihood subject to the Bellman equation as a constraint. MPEC still requires $P$.

\subsection{Hotz--Miller / Conditional Choice Probabilities (CCP)}
\label{sec:hotz-miller}

\citet{hotz1993conditional} attack the inner-loop bottleneck differently. Their starting point is the observation underpinning~\eqref{eq:Hotz-Miller-inversion}: the \emph{differences} of $Q^\star$ across actions at the same state are already pinned down by the observed log choice-probability ratios. Identity~\eqref{eq:HM-Q-at-anchor} rearranges to
\begin{equation}
V^\star(s) \;=\; Q^\star(s, a_s) - \log \pi^\star(a_s \mid s),
\label{eq:HM-V}
\end{equation}
and the linear fixed-point equation~\eqref{eq:V-fixedpoint} that follows is a Bellman-like equation for $V^\star$ alone, with reward $-\log \pi^\star(a_s \mid \cdot)$. Once $V^\star$ is in hand, $Q^\star(s, a)$ follows from~\eqref{eq:Hotz-Miller-inversion} and~\eqref{eq:HM-V}, and $r$ follows from the soft Bellman equation~\eqref{eq:r-recovery}.

\paragraph{Forward simulation.}
The expectation in~\eqref{eq:V-fixedpoint} can be approximated by forward Monte Carlo under the anchor-conditioned kernel: simulate trajectories of length $T$ starting at $s$, set the action at each simulated state $s_t$ to the anchor action $a_{s_t}$, draw $s_{t+1}$ from $\hat P(\cdot \mid s_t,a_{s_t})$, and accumulate the per-step pseudo-reward $-\log \hat\pi^\star(a_{s_t} \mid s_t)$. Drawing actions from $\hat\pi^\star$ would solve a different Bellman equation; the anchor equation~\eqref{eq:V-fixedpoint} specifically uses $P(\cdot\mid s,a_s)$. This sidesteps the high-dimensional fixed-point computation of NFXP entirely. Recursive variants of CCP \citep{aguirregabiria2002swapping,aguirregabiria2007sequential} interleave $\pi^\star$ updates and value iteration.

\paragraph{The statistical obstruction.}
We bought the computational savings of CCP at a price, and it is a steep one. Forward simulation requires \emph{estimating $\hat P$ in high dimensions}, then propagating its error over $T = \Theta(1/(1-\beta))$ steps. With a continuous or high-dimensional state space, $\hat P$ is itself subject to a curse of dimensionality: a nonparametric kernel density estimator of $P(s' \mid s, a)$ has rate $N^{-c/d}$ for some smoothness-dependent $c$. The statistical complexity blows up exactly where NFXP's computational complexity did.

\paragraph{Recovery of $r$.}
Once $V^\star$ and $Q^\star$ are recovered, the reward is read off~\eqref{eq:r-recovery}. This expression requires another $P$-expectation; in continuous states it is again subject to transition-estimation error.

\paragraph{The summary.}
NFXP and CCP are duals. NFXP fits $r$ in an outer loop and solves the Bellman fixed point in the inner. CCP fits $\pi^\star$ first and exploits Hotz--Miller inversion to avoid the inner fixed point but needs $P$ everywhere. \emph{Both require either knowing or estimating $P$, and both impose either computational ($|\cS| \cdot |\cA|$ DP table) or statistical ($\hat P$ in $d$ dimensions) burdens scaling exponentially in $d$.} Neither produces a single differentiable objective whose minimizer is $Q^\star$.

\subsection{Temporal-Difference DDC: Linear Semi-Gradient and Approximate Value Iteration}
\label{sec:td-ddc}

The natural next thought is: \emph{can we avoid the transition kernel by using temporal-difference updates based directly on observed transitions?} \citet{adusumilli2019temporal} propose exactly this. More precisely, they propose \emph{two} transition-density-free TD-style approaches for the recursive terms that enter CCP estimation:
\begin{enumerate}[leftmargin=2em,label=(\arabic*)]
\item \textbf{Linear semi-gradient TD}, which approximates the recursive terms with basis functions and updates the coefficients by stochastic approximation.
\item \textbf{Approximate Value Iteration (AVI)}, which builds a sequence of approximations by freezing the previous continuation-value estimate and solving a fresh nonparametric regression problem at each iteration.
\end{enumerate}
In the full DDC pseudo-likelihood these recursive terms are the continuation-value objects induced by Hotz--Miller inversion. In our anchor-action notation they collapse to the auxiliary value function in~\eqref{eq:V-fixedpoint}. Thus the same analysis covers both Adusumilli--Eckardt methods: one is the semi-gradient form of TD, and the other is fitted/approximate value iteration. The only twist is that we are now bootstrapping the auxiliary value function~\eqref{eq:V-fixedpoint}, not the original value function of an MDP.

\subsubsection{Linear Function Approximation and the Projected Bellman Equation}
\label{sec:linFA}

Fix a feature map $\phi : \cS \to \R^d$ and let $\Phi$ denote the $|\cS| \times d$ feature matrix (in the finite-state case; the general case is analogous in $L^2$). Parametrize $V_\theta(s) := \phi(s)^\top \theta$, $\theta \in \R^d$. The hypothesis class is the linear span $\cF := \{V : V(s) = \phi(s)^\top \theta,\;\theta \in \R^d\}$, a closed subspace of bounded functions on $\cS$.

Fix a distribution $\rho \in \Delta(\cS)$ (the data distribution) with associated $L^2(\rho)$-inner product $\langle f, g \rangle_\rho := \EE_{s \sim \rho}[f(s) g(s)]$ and norm $\|f\|_\rho^2 := \langle f, f\rangle_\rho$. The $L^2(\rho)$-projection onto $\cF$ is
\begin{equation*}
\Pi_{\cF, \rho} V \;:=\; \argmin_{V_\theta \in \cF} \|V - V_\theta\|_\rho^2,
\end{equation*}
which is linear in $V$ and a non-expansion in $\|\cdot\|_\rho$.

Let $\cT^{\pi_b}$ denote the (linear) Bellman expectation operator under a behavior policy $\pi_b$ with corresponding kernel $P^{\pi_b}(s' \mid s) = \sum_a \pi_b(a \mid s) P(s' \mid s, a)$:
\begin{equation}
\cT^{\pi_b} V(s) \;:=\; r^{\pi_b}(s) + \beta\,\EE_{s' \sim P^{\pi_b}(\cdot \mid s)}[V(s')], \qquad r^{\pi_b}(s) := \sum_a \pi_b(a \mid s) r(s, a).
\label{eq:bellman-expectation}
\end{equation}
For the TD-DDC application below, $\pi_b$ will be the anchor-conditioned chain $\pi_b(a \mid s) = \one[a = a_s]$ and the per-state ``reward'' will be $-\log \pi^\star(a_s \mid s)$.

\begin{lemma}[On-policy contraction]
\label{lem:on-pol-contract}
If $\rho$ is the stationary distribution of $P^{\pi_b}$, then $\cT^{\pi_b}$ is a $\beta$-contraction on $L^2(\rho)$, and $\Pi_{\cF, \rho}\cT^{\pi_b}$ is a $\beta$-contraction on $\cF$.
\end{lemma}

\begin{proof}
For any $V_1, V_2 \in L^2(\rho)$, since $\rho$ is stationary under $P^{\pi_b}$, $\rho P^{\pi_b} = \rho$. By Jensen's inequality applied conditionally,
\begin{equation*}
\|\cT^{\pi_b}V_1 - \cT^{\pi_b}V_2\|_\rho^2 \;=\; \beta^2 \EE_{s \sim \rho}\!\bigl[(\EE_{s' \sim P^{\pi_b}(\cdot \mid s)}[V_1(s') - V_2(s')])^2\bigr] \;\le\; \beta^2 \EE_{s' \sim \rho P^{\pi_b}}\!\bigl[(V_1(s') - V_2(s'))^2\bigr] \;=\; \beta^2 \|V_1 - V_2\|_\rho^2.
\end{equation*}
Since $\Pi_{\cF, \rho}$ is a non-expansion, the composition is also a $\beta$-contraction.
\end{proof}

The unique fixed point $V_{\cF, \rho}^\star$ of $\Pi_{\cF, \rho}\cT^{\pi_b}$ is the \emph{projected Bellman fixed point}. It solves the projected Bellman equation $V_{\cF, \rho} = \Pi_{\cF, \rho}\cT^{\pi_b} V_{\cF, \rho}$; in matrix form, $\theta^\star$ solves $A\theta = b$ with
\begin{equation}
A \;=\; \Phi^\top D_\rho (\Phi - \beta P^{\pi_b}\Phi), \qquad b \;=\; \Phi^\top D_\rho r^{\pi_b}.
\label{eq:A-b}
\end{equation}
When $\rho$ is the stationary distribution of $P^{\pi_b}$, \citet{tsitsiklis1996analysis} show that $A$ is positive definite. Linear TD$(0)$, whose update direction has expectation $b - A\theta$, then converges to $A^{-1} b$ with rate determined by the smallest eigenvalue of $A$.

\subsubsection{The Deadly Triad: Off-Policy Divergence of Linear TD}
\label{sec:deadly-triad}

The on-policy stability of Lemma~\ref{lem:on-pol-contract} relies critically on $\rho$ being the stationary distribution of $P^{\pi_b}$. The combined presence of (i) linear (or more general) function approximation, (ii) bootstrapping, and (iii) off-policy data is known as the \emph{deadly triad} \citep{sutton2018reinforcement,van2018deep}. Here is its TD-DDC face.

\begin{example}[Off-policy divergence with two states]
\label{ex:off-pol-div}
Take $\cS = \{1, 2\}$, $\cA = \{a\}$ (so $\pi_b$ is trivial), with deterministic transitions $1 \mapsto 2$ and $2 \mapsto 2$, discount $\beta = 0.99$, and reward $r \equiv 0$. The true value is $V^\star \equiv 0$. Choose feature map $\phi(1) = 1$, $\phi(2) = 2$, so $\cF$ is one-dimensional with parameter $\theta$. The stationary distribution of $P$ is concentrated at state $2$. Now suppose the data distribution gives most weight to state $1$: $\rho = (0.99, 0.01)$. Compute
\begin{align}
    A \;=\; \Phi^\top D_\rho (\Phi - \beta P\Phi) \;&=\; [1,\,2]\,\diag(0.99, 0.01)\,\bigl([1,\,2]^\top - 0.99 \cdot [2,\,2]^\top\bigr) \notag
    \\
    \;&=\; 0.99 \cdot 1 \cdot (1 - 1.98) + 0.01 \cdot 2 \cdot (2 - 1.98) \approx -0.97, \notag
\end{align}
so $A < 0$. The linear TD update direction is $\propto (b - A\theta) = -A\theta \approx 0.97\,\theta$: for any nonzero initialization and a fixed small positive stepsize, the iterates move away from zero and grow exponentially. The finite TD equilibrium is $\theta=0$, but it is a repelling equilibrium for this off-policy update, despite $V^\star \equiv 0$ being trivially representable.
\end{example}

This is not a pathological construction. \citet{wang2021instabilities} document the same phenomenon with neural-network function approximation, where it persists even with mild representation pre-training. And critically for us: the deadly triad applies in particular to the DDC application below. The offline DDC dataset is generated by the expert acting according to $\pi^\star$ under \emph{all} actions, so its state marginal is $d^{\pi^\star}_S$. The TD-DDC update, however, bootstraps along the \emph{anchor-conditioned} chain whose kernel is $P(\cdot \mid \cdot, a_s)$, a deterministic ``policy'' that always picks the state-specific anchor. After restricting to anchor transitions, the effective sampling distribution is $\rho_A(s) \propto \pi^\star(a_s \mid s)\, d^{\pi^\star}_S(s)$, which is in general \emph{not} the stationary distribution of $P(\cdot \mid \cdot, a_s)$. The on-policy contraction of Lemma~\ref{lem:on-pol-contract} therefore does not apply to this update, and Example~\ref{ex:off-pol-div}'s pathology is on the table.

\subsubsection{Approximate Value Iteration as Fitted Value Iteration}
\label{sec:td-ddc-avi}

The second Adusumilli--Eckardt method is not the same as the linear semi-gradient recursion above. It is \emph{Approximate Value Iteration}: freeze the current approximation inside the continuation term, then solve a new regression problem for the next approximation. In the linear case this takes the form
\begin{equation}
\theta^+ \;:=\; \argmin_\theta\;\EE_{s \sim \rho,\; s' \sim P^{\pi_b}(\cdot \mid s)}\!\bigl[(r^{\pi_b}(s) + \beta\,\phi(s')^\top \theta^- - \phi(s)^\top\theta)^2\bigr], \qquad \theta^- \leftarrow \theta^+,
\label{eq:fvi}
\end{equation}
and iterate. The inner problem is a well-conditioned least-squares problem whenever $\Phi^\top D_\rho \Phi \succ 0$; its closed form is
\begin{equation*}
\theta^+ \;=\; (\Phi^\top D_\rho \Phi)^{-1} \Phi^\top D_\rho (r^{\pi_b} + \beta P^{\pi_b}\Phi\theta^-) \;=\; (\Phi^\top D_\rho \Phi)^{-1} \Phi^\top D_\rho \cT^{\pi_b}(\Phi\theta^-).
\end{equation*}
Viewed as a map $\theta^- \mapsto \theta^+$, this is exactly \emph{fitted value iteration} (FVI): $V_{\theta^+} = \Pi_{\cF, \rho} \cT^{\pi_b} V_{\theta^-}$. In a nonparametric or sieve implementation, the same map is computed by replacing the linear least-squares problem with the corresponding regression learner over the chosen function class. Thus, in these notes, Adusumilli--Eckardt's AVI method is the DDC specialization of the standard fitted-value-iteration operator.

\begin{proposition}[FVI fixed point]
\label{prop:fvi}
Suppose $\Pi_{\cF, \rho}\cT^{\pi_b}$ is a contraction in $\|\cdot\|_\rho$ (for example, the on-policy case of Lemma~\ref{lem:on-pol-contract}, or any separately verified stability condition on the projected operator). Then FVI/AVI iterates~\eqref{eq:fvi} converge to a unique fixed point $V_{\cF,\rho}^\star = \Pi_{\cF, \rho}\cT^{\pi_b} V_{\cF,\rho}^\star$. In general, $V_{\cF,\rho}^\star \neq V^\star$.
\end{proposition}

\begin{proof}
The first claim follows from Banach's theorem applied to the FVI map on the closed subspace $\cF \subseteq L^2(\rho)$. For the second, note that any fixed point of $\Pi_{\cF, \rho}\cT^{\pi_b}$ lies in $\cF$, while $V^\star$, the unconstrained Bellman fixed point, may not. Concretely, $V^\star$ satisfies $V^\star = \cT^{\pi_b} V^\star$ but generally not $V^\star = \Pi_{\cF, \rho} V^\star$, so $\Pi_{\cF, \rho} \cT^{\pi_b} V^\star = \Pi_{\cF, \rho} V^\star \neq V^\star$ whenever $V^\star \notin \cF$. In particular, whenever the true value function does not lie in the chosen feature class, the FVI fixed point is the closest $\cF$-element to a transformed target and differs from $V^\star$ by the approximation error.
\end{proof}

\begin{corollary}[Projected fixed points and realizability]
\label{cor:closure}
Assume the projected operator $\Pi_{\cF,\rho}\cT^{\pi_b}$ has a unique fixed point. If $V^\star \in \cF$, then this fixed point equals $V^\star$, because $\Pi_{\cF,\rho}\cT^{\pi_b}V^\star = \Pi_{\cF,\rho}V^\star = V^\star$. If $V^\star \notin \cF$, the projected fixed point cannot equal $V^\star$, since every projected fixed point lies in $\cF$. Closure of $\cF$ under $\cT^{\pi_b}$ is a stronger sufficient condition for projection-free iterates, but it is not necessary for equality of the limiting projected fixed point. Thus the two distinct obstructions are stability of the projected operator and realizability of $V^\star$ in the chosen class.
\end{corollary}

\subsubsection{Application: The Two TD-DDC Methods of Adusumilli--Eckardt}
\label{sec:ae-application}

\citet{adusumilli2019temporal} apply this machinery in the DDC setting with the substitutions of Section~\ref{sec:hotz-miller}: $r^{\pi_b}(s) := -\log \pi^\star(a_s \mid s)$ and $\pi_b(a \mid s) := \one[a = a_s]$. Let
\begin{equation*}
I_A \;:=\; \{i \in \{1,\ldots,N\}: a_i = a_{s_i}\}
\end{equation*}
be the set of observed transitions in which the expert took the anchor action, and write $\cD_A := \{(s_i, s_i') : i \in I_A\}$ for the corresponding sub-dataset with state marginal $\rho_A$.

\paragraph{Method 1: linear semi-gradient TD.}
The first method updates the coefficient vector using the current parameter on both sides of the TD error:
\begin{equation}
\theta_{t+1} \;=\; \theta_t + \alpha_t\,\bigl(-\log \pi^\star(a_{s_t} \mid s_t) + \beta\,\phi(s_{t+1})^\top \theta_t - \phi(s_t)^\top \theta_t\bigr)\,\phi(s_t),
\label{eq:ae-semigrad}
\end{equation}
sampled from transitions $(s_t, a_t = a_{s_t}, s_{t+1})$ in $\cD$.

\paragraph{Method 2: approximate value iteration.}
The second method freezes the previous value approximation. Starting from $V_0 \in \cF$, define pseudo-outcomes for $i \in I_A$ by
\begin{equation}
Y_i^{(k)} \;:=\; -\log \pi^\star(a_{s_i} \mid s_i) + \beta V_k(s_i'),
\label{eq:ae-avi-target}
\end{equation}
and compute
\begin{equation}
V_{k+1} \;\in\; \argmin_{f \in \cF}\;\frac{1}{|I_A|}\sum_{i \in I_A}\bigl(Y_i^{(k)} - f(s_i)\bigr)^2.
\label{eq:ae-avi-regression}
\end{equation}
When $\cF$ is linear, $V_k(s)=\phi(s)^\top\theta_k$ and~\eqref{eq:ae-avi-regression} is exactly~\eqref{eq:fvi}; when $\cF$ is a sieve or nonparametric class,~\eqref{eq:ae-avi-regression} is the nonparametric AVI regression. Either way, the population operator is $V_{k+1}=\Pi_{\cF,\rho_A}\cT^{\pi_b}V_k$.

\paragraph{Replaying the analysis in the DDC setting.}
The two methods have different failure modes.
\begin{itemize}[leftmargin=2em]
\item The \emph{linear semi-gradient} update~\eqref{eq:ae-semigrad} has all three deadly-triad ingredients: function approximation, bootstrapping, and off-policy state weighting relative to the anchor-conditioned chain. By Example~\ref{ex:off-pol-div}, it has no general guarantee of convergence to $V^\star$, the unique solution of~\eqref{eq:V-fixedpoint}.
\item The \emph{AVI/FVI} update~\eqref{eq:ae-avi-regression} avoids the within-iteration moving-target gradient problem by freezing $V_k$ and solving a supervised regression. But it is still an iterated projected Bellman method. If $\Pi_{\cF,\rho_A}\cT^{\pi_b}$ is a contraction, Proposition~\ref{prop:fvi} gives convergence; the limit is the projected fixed point $V_{\cF,\rho_A}^\star$, which equals $V^\star$ when $V^\star$ is realizable in $\cF$ and the projected fixed point is unique, as in Corollary~\ref{cor:closure}. Without such a contraction, even the projected iteration can fail to converge.
\end{itemize}

\paragraph{The honest scorecard.}
Both Adusumilli--Eckardt methods avoid the explicit transition kernel that CCP/Hotz--Miller needed: they use observed $(s_t, s_{t+1})$ pairs in place of an explicit $\EE_{P(\cdot \mid s, a)}$. That is a genuine improvement over forward simulation. But the price is not the same for the two methods. Tallying:
\begin{itemize}[leftmargin=2em]
\item NFXP: statistically efficient, transition-kernel-needing, computationally cursed in $d$.
\item CCP/Hotz--Miller: computationally cheap, transition-estimation-needing, statistically cursed in $d$.
\item Adusumilli--Eckardt linear semi-gradient TD: transition-kernel-free, but deadly-triad-vulnerable.
\item Adusumilli--Eckardt AVI/FVI: transition-kernel-free and stable as a sequence of regression problems, but still a bootstrapped projected-fixed-point method; convergence to the true $V^\star$ requires both stability of the projected operator and realizability of $V^\star$ in the chosen function class.
\end{itemize}
None produces a \emph{single} unbiased objective whose global minimum is the true $Q^\star$, computable with mild conditioning assumptions on the parametrization and without specifying $P$.

\subsubsection{Why ``Neural Orthogonalization'' (Target Networks) Does Not Resolve It}
\label{sec:td-ddc-target}

A reader familiar with deep RL might wonder whether the linear semi-gradient instability of Method~1 can be cured by introducing a target network, exactly as in the forward-RL setting. The answer is the same here as there, for the same reason: target networks change \emph{which} method you are running, not where it converges. We make this precise, then explain why it leaves the underlying issue intact.

\begin{proposition}[Target-stabilized semi-gradient TD-DDC equals AVI]
\label{prop:td-ddc-target}
Consider Method~1, equation~\eqref{eq:ae-semigrad}, modified so that the parameter $\theta$ inside the bootstrap term $\phi(s_{t+1})^\top \theta$ is replaced by a target parameter $\theta^-$, refreshed only after the inner optimization in $\theta$ has reached its empirical minimum at each outer iteration. Then the outer iterates coincide with Method~2, equation~\eqref{eq:ae-avi-regression}, in both empirical and population form:
\begin{equation}
V_{\theta^{(k+1)}} \;=\; \widehat\Pi_{\cF, \cD_A}\bigl[\hat\cT^{\pi_b} V_{\theta^{(k)}}\bigr] \quad\text{(empirical)}, \qquad
V^{(k+1)} \;=\; \Pi_{\cF, \rho_A}\,\cT^{\pi_b} V^{(k)} \quad\text{(population)},
\label{eq:target-fvi-equivalence}
\end{equation}
where $\widehat\Pi_{\cF, \cD_A} f := \argmin_{V \in \cF}\,\widehat\EE_{\cD_A}[(f - V)^2]$ is the empirical $L^2(\cD_A)$-projection and $\hat\cT^{\pi_b}V(s, s') := -\log \pi^\star(a_s \mid s) + \beta V(s')$ is the sample version of $\cT^{\pi_b}$.
\end{proposition}

\begin{proof}
The argument is direct, and is especially clean because $\cT^{\pi_b}$ is a Bellman \emph{expectation} operator and there is no $\max$ in sight.

\emph{Empirical step.} Fix the outer iterate $\theta^{(k)}$ and set $\theta^- := \theta^{(k)}$ throughout the inner loop. The per-sample target
\begin{equation*}
y_i \;:=\; -\log \pi^\star(a_{s_i} \mid s_i) + \beta\,\phi(s_i')^\top \theta^{(k)} \;=\; \hat\cT^{\pi_b} V_{\theta^{(k)}}(s_i, s_i')
\end{equation*}
does not depend on the optimization variable $\theta$. The inner loss $\ell_i(\theta) := \tfrac{1}{2}(y_i - \phi(s_i)^\top \theta)^2$ has gradient $-(y_i - \phi(s_i)^\top \theta)\,\phi(s_i)$, which is exactly the negative of the target-stabilized version of the update direction in~\eqref{eq:ae-semigrad} (the bootstrap term is constant in $\theta$, so contributes nothing). Inner convergence yields
\begin{equation*}
\theta^{(k+1)} \;\in\; \argmin_\theta\;\widehat\EE_{\cD_A}\!\bigl[(\hat\cT^{\pi_b} V_{\theta^{(k)}} - \phi(s)^\top \theta)^2\bigr],
\end{equation*}
which is the empirical regression~\eqref{eq:ae-avi-regression} defining Method~2.

\emph{Population step.} Since $\cT^{\pi_b}$ is the Bellman \emph{expectation} operator under the anchor-conditioned chain, $\EE_{s' \sim P^{\pi_b}(\cdot \mid s)}[\hat\cT^{\pi_b} V_{\theta^{(k)}}(s, s')] = \cT^{\pi_b} V_{\theta^{(k)}}(s)$. The inner objective splits as
\begin{equation*}
\EE_{\rho_A,\, P^{\pi_b}}\!\bigl[(\hat\cT^{\pi_b} V_{\theta^{(k)}} - \phi(s)^\top \theta)^2\bigr] \;=\; \EE_{\rho_A}\!\bigl[(\cT^{\pi_b} V_{\theta^{(k)}} - \phi(s)^\top \theta)^2\bigr] \;+\; \beta^2\,\EE_{\rho_A}\!\bigl[\Var_{s' \mid s}(V_{\theta^{(k)}}(s'))\bigr],
\end{equation*}
by the variance decomposition (cf.\ Lemma~\ref{lem:double-sampling}). The second term depends only on $\theta^{(k)}$, not on $\theta$. So the minimizer in $\theta$ is $\Pi_{\cF, \rho_A}\, \cT^{\pi_b} V_{\theta^{(k)}}$, exactly Method~2's population operator.
\end{proof}

\paragraph{What the proposition says, and what it doesn't.}
Target stabilization moves Method~1 onto Method~2's orbit; it does not change Method~2's destination. The projected-fixed-point obstruction of Corollary~\ref{cor:closure} is unaffected: even with perfect target stabilization and idealized inner optimization, the limit (when it exists) is the projected fixed point $V_{\cF, \rho_A}^\star$. That limit equals the true auxiliary $V^\star$ only under additional conditions such as realizability of $V^\star$ in $\cF$ and uniqueness/stability of the projected fixed point. These conditions are not innocuous for neural function classes in continuous-state, off-policy settings; they are exactly the kind of structural assumptions the deadly-triad analysis (Section~\ref{sec:deadly-triad}) warns us we cannot expect for free.

In particular, three caveats are worth recording:
\begin{itemize}[leftmargin=2em]
\item \emph{Linear $\cF$, off-policy data} ($\rho_A$ not invariant under $P^{\pi_b}$): the projected operator can fail to be a contraction; Example~\ref{ex:off-pol-div} applies, and target stabilization does not rescue it.
\item \emph{Linear $\cF$, on-policy stationary data}: the projected operator is a $\beta$-contraction by Lemma~\ref{lem:on-pol-contract}, and target-stabilized Method~1 converges to $V_{\cF, \rho_A}^\star$, the projected fixed point. This limit equals $V^\star$ if $V^\star$ lies in $\cF$; otherwise it is a projected approximation.
\item \emph{Nonlinear $\cF$} (the relevant regime for offline IRL/DDC at scale): $\Pi_{\cF, \rho_A}$ is set-valued, the inner regression may have multiple local minima, and \eqref{eq:target-fvi-equivalence} does not uniquely define a single-valued map. Even pointwise selecting an inner minimizer, $\Pi_{\cF, \rho_A}\,\cT^{\pi_b}$ need not be Lipschitz, and convergence requires assumptions far beyond what is verifiable in practice.
\end{itemize}

The methodological takeaway is the following:
\begin{quote}
\emph{Target networks stabilize the trajectory of bootstrapped TD-DDC updates but do not change their destination. The destination is the projected Bellman fixed point of the auxiliary value function~\eqref{eq:V-fixedpoint}, not the auxiliary value function itself.}
\end{quote}

This rules out the most familiar deep-RL fix and points to the same conclusion: the right escape route in offline IRL/DDC is to abandon bootstrapping entirely in favor of direct (bias-corrected) Bellman-residual minimization. That is the gap Section~\ref{sec:erm} closes.

\section{Modern IRL Methods}
\label{sec:irl}

We now turn to the ML/IRL literature. Four families dominate it: adversarial IRL, occupancy matching, IQ-Learn, and offline maximum-likelihood IRL. For each we (a) write out the actual training objective, (b) state precisely what it identifies and what it does not, and (c) compare to the anchor-action route. A general survey of imitation learning is given by \citet{zare2024survey}.

A preview of where the section lands: each modern method makes a different bet against the same identification ambiguity (Lemma~\ref{lem:nonid}), and each pays for the bet somewhere: in the form of reward identified, in stability, in compatibility with offline data, or in what the method actually optimizes.

\subsection{Adversarial IRL (AIRL)}
\label{sec:airl}

Section~\ref{sec:airl-id} introduced AIRL's language and formulation: the potential-shaped discriminator~\eqref{eq:airl-discriminator} and the alternating saddle update~\eqref{eq:airl-disc-update}. We return to AIRL here as a modern IRL method, focusing on what that formulation delivers computationally and what it still cannot identify.

\paragraph{What the saddle delivers.}
The textbook saddle calculation needs two idealizations made explicit. First, the policy player must have converged to the expert, so $d^\pi = d^{\pi^\star}$ on the relevant support. Second, the AIRL discriminator class must be able to realize its Bayes optimum. Under both, $D^\star \equiv \tfrac12$ on the matched support, and since~\eqref{eq:airl-discriminator} gives $D = \tfrac12$ exactly when $\exp f_{\theta, \phi} = \pi(a \mid s)$, the saddle pins down
\begin{equation*}
f_{\theta, \phi}(s, a, s') = \log \pi^\star(a \mid s) = A^\star(s, a).
\end{equation*}
As emphasized in Section~\ref{sec:airl-id}, this is an \emph{advantage} identity, not yet a reward identity, and it carries no information beyond the expert's conditional choice probabilities, the same object Hotz--Miller inversion~\eqref{eq:Hotz-Miller-inversion} consumes through log-probability differences. The discriminator's real burden is therefore not to discover the advantage, but to split it into a reward part and a shaping part. In stochastic environments even the equality above is delicate: the right-hand side $A^\star(s, a)$ is independent of $s'$ while the AIRL class $g(s) + \beta h(s') - h(s)$ is not, so exact $D \equiv \tfrac12$ within the parametrized class may be unattainable. Section~\ref{sec:airl-id} treats this carefully through the population moment~\eqref{eq:airl-expectation-condition}.

\paragraph{Identification limitations.}
The self-contained analysis of Section~\ref{sec:airl-id} (Proposition~\ref{prop:airl-deterministic-decomp}) shows precisely what extra structure makes reward recovery valid, and what goes wrong otherwise. The limitations are structural rather than merely computational:
\begin{itemize}[leftmargin=2em]
\item With a state-only ground-truth reward, a state-only reward head, deterministic dynamics, an attained exact saddle, and the decomposable-coverage condition, AIRL recovers the reward up to an additive constant: $g_\theta(s) = r(s) + C$ and $h_\phi(s) = V^\star(s) + C/(1 - \beta)$.
\item If $g_\theta$ is allowed to depend on $(s, a)$, the potential-shaped family~\eqref{eq:airl-state-action-shaping} gives the same advantage, so state-action rewards remain unidentified beyond the baseline shaping equivalence of Lemma~\ref{lem:nonid}.
\item In stochastic environments there is both a realizability gap, because the parametrized logit cannot match the $s'$-independent target $A^\star(s, a)$ pointwise, and an identification gap: even the population moment~\eqref{eq:airl-expectation-condition} needs a transition-kernel completeness condition to pin down the shaping potential.
\end{itemize}
The anchor-action route gives strictly more for these notes: it identifies state-action $r$ on the expert support without requiring the reward to be state-only and without relying on the AIRL decomposition (cf.\ Theorem~\ref{thm:magnac-thesmar} and Theorem~\ref{thm:cao-id}).

\paragraph{Optimization limitations.}
On top of this identification ceiling, AIRL inherits the well-documented instabilities of GAN training: the alternating minimax updates on $(\theta, \phi)$ and $\pi$ can fail to converge, exhibit mode collapse, or settle at a non-optimal saddle. And AIRL is fundamentally an \emph{online} method: the discriminator update~\eqref{eq:airl-disc-update} requires sampling from the current policy $\pi$, which is exactly what offline IRL forbids.

\subsection{Occupancy Matching: GAIL}
\label{sec:occ-matching}

\citet{ho2016generative} introduced \emph{Generative Adversarial Imitation Learning} (GAIL). The starting point is the occupancy-matching objective
\begin{equation*}
\min_\pi \;D_{\mathrm{JS}}\bigl(d^\pi \,\|\, d^{\pi^\star}\bigr),
\end{equation*}
where $D_{\mathrm{JS}}$ is the Jensen--Shannon divergence. The variational identity behind GAIL is the following. For any two distributions $p, q$ on a common space $\cX$, the GAN value at the inner $D$-optimum is
\begin{equation}
\max_{D : \cX \to (0, 1)} \;\EE_p[\log D] + \EE_q[\log(1 - D)]
\;=\; -\log 4 + 2\,D_{\mathrm{JS}}(p \,\|\, q),
\label{eq:gan-dual}
\end{equation}
with optimum attained at $D^\star = p / (p + q)$ \citep[Prop.~1]{goodfellow2020generative}. The derivation is direct: writing the integrand at a point as $V(D) := p \log D + q \log(1 - D)$ and setting $\partial V/\partial D = 0$ gives $D^\star = p/(p+q)$, and substituting back recovers $-\log 4 + 2D_{\mathrm{JS}}(p \| q)$ via $D_{\mathrm{JS}}(p\|q) = \tfrac{1}{2}\KL(p \| (p+q)/2) + \tfrac{1}{2}\KL(q \| (p+q)/2)$. So minimizing the JS over $\pi$ is equivalent (up to additive constants) to the saddle
\begin{equation}
\min_\pi\,\max_{D}\;\EE_{(s, a) \sim d^{\pi^\star}}[\log D(s, a)] + \EE_{(s, a) \sim d^\pi}[\log(1 - D(s, a))].
\label{eq:gail-saddle}
\end{equation}
GAIL solves~\eqref{eq:gail-saddle} with a GAN-style optimizer over a parametrized $D_\phi$ and $\pi_\theta$, using on-policy policy-gradient methods such as PPO \citep{schulman2017proximal,raffin2021stable} for the $\pi_\theta$ update. Related state-marginal-matching variants restrict $D$ to depend only on $s$ \citep{ni2021f}.

\paragraph{What GAIL is and isn't.}
GAIL is fundamentally an \emph{imitation} algorithm: its goal is to make $d^{\pi_\theta}$ close to $d^{\pi^\star}$ in JS divergence, and it succeeds at this on many tasks. But the crucial point is that \emph{it is not a reward-recovery algorithm}. At the saddle, the discriminator satisfies
\begin{equation*}
\log\frac{D^\star(s, a)}{1 - D^\star(s, a)} \;=\; \log \frac{d^{\pi^\star}(s, a)}{d^{\pi_\theta}(s, a)},
\end{equation*}
which is a log Radon--Nikodym derivative of expert vs.\ current-policy occupancies, not a reward. There is no way to extract $r$ from $D^\star$ without further identification assumptions of exactly the kind we already discussed (Theorem~\ref{thm:cao-id} and the anchor-action assumption).

\paragraph{Bellman consistency is not built in.}
The limitation can be seen directly from a one-line Bellman-flow identity. Fix any policy $\pi$, reward $r$, and candidate state-action value function $Q$. Define the policy Bellman operator
\begin{equation*}
(\cT_r^\pi Q)(s,a)
:=r(s,a)+\beta\EE_{s'\sim P(\cdot\mid s,a),\,a'\sim\pi(\cdot\mid s')}[Q(s',a')].
\end{equation*}
Using the normalized discounted occupancy
\begin{equation*}
d^\pi(s,a)=(1-\beta)\sum_{t\ge 0}\beta^t\PP_\pi(s_t=s,a_t=a),
\end{equation*}
we have
\begin{align}
\EE_{d^\pi}\!\left[Q(s,a)-\beta\EE[Q(s',a')\mid s,a]\right]
&=(1-\beta)\EE_{s_0\sim\nu_0,\,a_0\sim\pi(\cdot\mid s_0)}[Q(s_0,a_0)],
\label{eq:gail-flow-identity}\\
\EE_{d^\pi}\!\left[Q(s,a)-(\cT_r^\pi Q)(s,a)\right]
&=(1-\beta)\EE[Q(s_0,a_0)]-\EE_{d^\pi}[r(s,a)].
\label{eq:gail-weighted-bellman}
\end{align}
The first line is just telescoping: the second term equals $(1-\beta)\sum_{t\ge 0}\beta^{t+1}\EE[Q(s_{t+1},a_{t+1})]$, so subtracting it from $(1-\beta)\sum_{t\ge 0}\beta^t\EE[Q(s_t,a_t)]$ leaves only the $t=0$ term. The second line subtracts the reward.

Equation~\eqref{eq:gail-weighted-bellman} is a \emph{single weighted average} of Bellman residuals under the occupancy $d^\pi$. Matching $d^\pi$ to $d^{\pi^\star}$ can at most match such occupancy-weighted averages. It does not imply the pointwise Bellman equation
\begin{equation*}
Q(s,a)=(\cT_r^\pi Q)(s,a)\qquad\text{for every }(s,a).
\end{equation*}
Residuals can be positive on one part of the support and negative on another while their $d^\pi$-average remains zero. Thus occupancy matching places a distributional constraint on \emph{where the policy goes}; it does not place an equation-of-motion constraint on a candidate $Q$.

This also explains why a reward cannot be recovered from GAIL by first extracting some $Q$ and then applying the Bellman rearrangement $r=Q-\beta\EE[V_Q(s')\mid s,a]$. At a perfect GAIL match the discriminator itself is uninformative ($D^\star=1/2$), and any separately fitted $Q$ is unconstrained by the pointwise Bellman residual unless that residual is explicitly included in the objective. The anchor-action and ERM routes of Section~\ref{sec:anchor} and Section~\ref{sec:erm} do exactly that: they enforce the Bellman equation at the anchor action rather than hoping it follows from occupancy matching.

\paragraph{Offline limitations.}
GAIL's policy update is on-policy: it requires sampling from $\pi_\theta$ to compute $\EE_{d^{\pi_\theta}}[\log(1 - D)]$. An offline variant, ValueDice \citep{kostrikov2019imitation}, re-derives the objective using off-policy distribution-ratio estimation via the Fenchel dual of the $f$-divergence: instead of sampling from $\pi_\theta$, it estimates the ratio $d^{\pi_\theta}/d^{\pi^\star}$ from offline data and substitutes it into the GAIL objective. The empirical performance of these off-policy corrections is mixed; \citet{li2022rethinkingvaluedice} show that improvements over a strong behavior-cloning baseline can be entirely attributed to architectural and regularization differences, not the off-policy correction itself. The deeper issue is the curse of horizon all over again: trajectory importance sampling for $T$-step corrections has variance $\Theta((\rho_{\max})^T)$ where $\rho_{\max}$ is the per-step importance ratio bound and $T \sim 1/(1-\beta)$ is the effective horizon, exponentially bad in the horizon for any non-trivial policy shift, as standard importance-sampling calculations show.

\subsection{IQ-Learn}
\label{sec:iqlearn}

\citet{garg2021iq} propose \emph{IQ-Learn} (Inverse soft-Q Learning), which sidesteps both forward simulation and adversarial training by changing variables from $(r, \pi)$ to $(Q, \pi)$. The key identity is the soft Bellman equation~\eqref{eq:soft-bellman}, rearranged:
\begin{equation}
r_Q(s, a) \;:=\; Q(s, a) - \beta\,\EE_{s' \sim P(\cdot \mid s, a)}[V_Q(s')], \qquad V_Q(s) := \LSE(Q(s, \cdot)).
\label{eq:iql-reward-recovery}
\end{equation}
Given any $Q$, equation~\eqref{eq:iql-reward-recovery} \emph{defines} a reward $r_Q$ for which $Q$ satisfies the soft Bellman equation by construction. IQ-Learn parametrizes $Q_\theta$ directly and maximizes
\begin{equation}
\cJ(Q_\theta) \;:=\; \EE_{(s, a) \sim d^{\pi^\star}}[r_Q(s, a)] \;-\; (1 - \beta)\,\EE_{s_0 \sim \nu_0}[V_Q(s_0)]
\;-\; \lambda\,\EE_{(s, a) \sim d^{\pi^\star}}[\psi(r_Q(s, a))],
\label{eq:iql-objective}
\end{equation}
where $\psi$ is a strongly convex regularizer (the chi-squared $\psi(x) = x^2/4$ in the simplest version) and $\lambda > 0$. To see why the first two terms behave as an entropy-regularized return gap, use the Bellman-flow identity: for any function $V$ and any policy $\pi$,
\begin{equation}
\EE_{(s, a) \sim d^\pi}\!\bigl[V(s) - \beta\,\EE_{s' \sim P(\cdot|s, a)}[V(s')]\bigr] \;=\; (1 - \beta)\,\EE_{s_0 \sim \nu_0}[V(s_0)],
\label{eq:bellman-flow}
\end{equation}
which is the discounted-occupancy form of the telescoping sum $\sum_t \beta^t(V(s_t) - \beta V(s_{t+1}))$. Applying~\eqref{eq:bellman-flow} with $V = V_Q$ and $\pi = \pi^\star$, and using the definition of $r_Q$,
\begin{equation*}
\EE_{(s, a) \sim d^{\pi^\star}}[r_Q(s, a)] \,-\, (1-\beta)\EE_{s_0 \sim \nu_0}[V_Q(s_0)] \;=\; \EE_{(s, a) \sim d^{\pi^\star}}\!\bigl[Q(s, a) - V_Q(s)\bigr].
\end{equation*}
The RHS is exactly $\EE_{(s, a) \sim d^{\pi^\star}}[\log \hat p_Q(a \mid s)] = -\EE_{s \sim d^{\pi^\star}_S}\!\bigl[\KL(\pi^\star(\cdot|s)\,\|\,\hat p_Q(\cdot|s))\bigr] - \EE_{s \sim d^{\pi^\star}_S}[\Ent(\pi^\star(\cdot|s))]$ by the same decomposition as in Lemma~\ref{lem:nll}. So maximizing the first two terms of $\cJ$ in $Q$ is exactly minimizing the KL between the softmax of $Q$ and the expert policy, i.e., it fits $Q$ to a value function whose softmax matches $\pi^\star$, hence to the soft Bellman equation of $r_Q$ by construction. The chi-squared term keeps $r_Q$ from blowing up.

\paragraph{What IQ-Learn does and does not pin down.}
The first two terms of~\eqref{eq:iql-objective} do not fix the state-only shaping ambiguity. Let $c:\cS\to\R$ be bounded and define $Q_c(s,a):=Q(s,a)+c(s)$. Then
\begin{equation*}
V_{Q_c}(s)=V_Q(s)+c(s),\qquad
r_{Q_c}(s,a)=r_Q(s,a)+c(s)-\beta \EE[c(s')\mid s,a].
\end{equation*}
The change in the reward term is, by the Bellman-flow identity~\eqref{eq:bellman-flow},
\begin{equation*}
\EE_{d^{\pi^\star}}[r_{Q_c}-r_Q]
=\EE_{d^{\pi^\star}}\bigl[c(s)-\beta\EE[c(s')\mid s,a]\bigr]
=(1-\beta)\EE_{\nu_0}[c(s_0)].
\end{equation*}
The initial-value term changes by exactly the same amount with the opposite sign:
\begin{equation*}
-(1-\beta)\EE_{\nu_0}[V_{Q_c}(s_0)-V_Q(s_0)]
=-(1-\beta)\EE_{\nu_0}[c(s_0)].
\end{equation*}
Therefore the first two terms of IQ-Learn are invariant under \emph{every} bounded state-only shift $c$, not merely under a scalar normalization. The regularizer $-\lambda\EE[\psi(r_Q)]$ is the only part of~\eqref{eq:iql-objective} that selects a representative within the potential-shaped equivalence class, and that selection is a penalty preference rather than statistical identification of the ground-truth reward. Thus IQ-Learn can fit a softmax policy and define a Bellman-consistent reward by construction, but without an additional normalization such as the anchor action it does not identify a unique state-action reward. This is materially weaker than the anchor-action identification of Theorem~\ref{thm:magnac-thesmar}, which fixes the shaping function rather than choosing a regularizer-preferred representative.

\subsection{Offline Maximum-Likelihood IRL (Zeng et al.)}
\label{sec:offline-mlirl}

\citet{zeng2023understanding} propose \emph{offline ML-IRL}. Stripped of vocabulary, the method is essentially a parametric, neural re-skinning of Hotz--Miller / CCP (Section~\ref{sec:hotz-miller}) with two modifications: (a)~the closed-form auxiliary inversion of Section~\ref{sec:hotz-miller} is replaced by an iterated forward-soft-RL inner loop with an explicit reward parametrization $r_\theta$, driven by an outer softmax-MLE objective; and (b)~a conservatism penalty, analogous to Conservative Q-Learning (CQL) of \citet{kumar2020conservative}, is added inside the soft-RL inner loop to handle distribution shift. Related conservative offline IRL methods appear in \citet{yue2023clare}.

It is worth making the structural parallel to Hotz--Miller explicit, because it is the right way to understand what the method does and does not buy. Both methods (i)~estimate $\hat P$ from $\cD$, and (ii)~alternate between a forward step that produces a value function from $\hat P$ together with the current reward (or current value) and a fitting step that updates the reward (or the choice probabilities). CCP performs the forward step in closed form by solving the auxiliary linear fixed-point equation~\eqref{eq:V-fixedpoint} and reading off $r$ from~\eqref{eq:r-recovery}; Zeng et al.\ instead solve a parametric soft-RL problem at each outer step. The substrate is the same; only the parametrization and the objective differ.

The procedure is a doubly-nested alternation. At outer iteration $k$ with current reward parameter $\theta_k$:
\begin{enumerate}[leftmargin=2em,label=(\arabic*)]
\item Fit a transition model $\hat P_\phi$ from $\cD$ by maximum likelihood (this is done once or refreshed periodically).
\item Inner loop: solve the soft-RL problem under $(r_{\theta_k}, \hat P_\phi)$ for the soft optimal $\pi_k$ via soft value iteration or actor--critic with a CQL-style penalty added to the value function to discourage extrapolation beyond the support of $\cD$.
\item Outer update: $\theta_{k+1} \leftarrow \theta_k + \eta\, \nabla_\theta \widehat\EE_\cD[\log \hat p_{Q_{\theta_k}}(a \mid s)]$, where $Q_{\theta_k}$ is the soft $Q$-function corresponding to $\pi_k$.
\item Repeat until convergence.
\end{enumerate}
The conservatism penalty is the one substantive new ingredient, and it is genuinely useful: it handles the unvisited states that vanilla CCP could not, by discouraging the inner-loop value from extrapolating wildly outside the support of $\cD$. Everything else inherits CCP's structure, as well as CCP's limitations:
\begin{itemize}[leftmargin=2em]
\item It needs to estimate $\hat P_\phi$ in high dimensions, with the same statistical curse-of-dimensionality as CCP/Hotz--Miller.
\item The inner-loop policy fit must contend with the bias of $\hat P_\phi$ and the distribution-shift concerns that motivate the conservatism penalty.
\item The outer loop is nested optimization, which is the very thing we have been trying to avoid since Rust.
\end{itemize}
\citet{zeng2023understanding} prove finite-time guarantees of the form $\|\widehat r - r\|_2^2 \le \cO(N^{-1/2}) + \text{(model error)}$ under linear-MDP assumptions and bounded transition-model error, but the analysis does not extend directly to neural function approximation without further structure.

The honest takeaway is that offline ML-IRL inherits Hotz--Miller's strengths and weaknesses essentially unchanged; conservatism is orthogonal to the underlying estimator and could in principle be bolted onto \emph{any} method in this section, including the ERM framework of Section~\ref{sec:erm}.

\paragraph{Where modern IRL leaves us.}
None of the methods reviewed in this section gives all of the following at once:
\begin{itemize}[leftmargin=2em]
\item[(a)] A single (non-nested) gradient-based objective.
\item[(b)] Provable global convergence in nonlinear function classes.
\item[(c)] State-action reward recovery in continuous-state, stochastic-transition environments.
\item[(d)] No transition-kernel estimation.
\end{itemize}
The next section gives one route through this gap under explicit anchor, empirical Jacobian-conditioning/PL, realizability, and stability assumptions.

\section{The Empirical Risk Minimization Approach}
\label{sec:erm}

We finally reach the ERM framework of \citet{kang2025empirical}.  The easiest way to understand the method is not to start with the minimax objective.  Instead, start from the two restrictions that have appeared throughout the notes:
\begin{enumerate}[leftmargin=2em,label=(\roman*)]
\item the expert's observed choices identify the \emph{within-state} differences of $Q^\star$ through the softmax likelihood;
\item the anchor-action Bellman equation fixes the remaining state-wise level of $Q^\star$.
\end{enumerate}
The ERM estimator is simply the one-shot version of these two restrictions.  It parametrizes $Q$ directly, rather than parametrizing $r$ and repeatedly solving the forward Bellman equation.  This is the step that removes the nested $r\mapsto Q$ optimization that classical DDC methods require.

Throughout this section, $a_s$ denotes the anchor action at state $s$, and
\[
    r_A(s):=r(s,a_s)
\]
denotes its known reward.  The normalized presentation often sets $r_A(s)=0$ for all $s$; the formulation below allows any known anchor payoff, which is the notation used in the paper.

The build-up is as follows.  Section~\ref{sec:erm-from-classical-ddc} rewrites the classical identification equations as losses.  Section~\ref{sec:erm-objective} combines them into the population expected risk.  Section~\ref{sec:erm-td-failure} explains why the empirical problem cannot be obtained by naively replacing the Bellman expectation by one next-state sample.  Section~\ref{sec:erm-bias-correction} introduces the conditional-mean correction that removes that bias.  Sections~\ref{sec:gladius} and~\ref{sec:erm-pl} state the resulting GLADIUS algorithm and its conditional convergence guarantees.

\subsection{From classical DDC equations to two losses}
\label{sec:erm-from-classical-ddc}

Recall the two soft optimality equations.  If
\[
    V_Q(s):=\LSE(Q(s,\cdot))=\log\sum_{b\in\cA}\exp(Q(s,b)),
    \qquad
    \hat p_Q(a\mid s):={\exp(Q(s,a))\over\sum_{b\in\cA}\exp(Q(s,b))},
\]
then the expert policy and the optimal $Q$ satisfy
\begin{align}
    \pi^\star(a\mid s)&=\hat p_{Q^\star}(a\mid s),
    \label{eq:erm-softmax-restriction}\\
    Q^\star(s,a)&=r(s,a)+\beta\EE[V_{Q^\star}(s')\mid s,a].
    \label{eq:erm-bellman-restriction}
\end{align}
The first equation is what likelihood-based estimation uses.  The second equation is what dynamic programming, CCP, and Bellman-residual methods try to enforce.

Under the anchor-action assumption, the key Magnac--Thesmar observation is that we do \emph{not} need the Bellman equation for every action in order to identify $Q^\star$.  The likelihood condition~\eqref{eq:erm-softmax-restriction} identifies all within-state differences $Q^\star(s,a)-Q^\star(s,b)$.  It is only missing one scalar level per state.  The anchor-action Bellman equation supplies exactly that missing scalar:
\begin{equation}
    Q^\star(s,a_s)=r_A(s)+\beta\EE[V_{Q^\star}(s')\mid s,a_s].
    \label{eq:erm-anchor-bellman}
\end{equation}
This is the bridge from classical DDC to ERM.  Rust's nested fixed-point method enforces the full Bellman equation after parametrizing $r$; Magnac--Thesmar identification says that, once the anchor is fixed, the likelihood plus the anchor-action Bellman equation is enough.  Therefore, if we parametrize $Q$ directly, the two identifying restrictions can be expressed as two loss-minimization problems.

The first loss is the negative log-likelihood,
\begin{equation}
    \cL_{\mathrm{NLL}}(Q)(s,a):=-\log \hat p_Q(a\mid s).
    \label{eq:erm-nll-loss}
\end{equation}
By Lemma~\ref{lem:nll}, its population minimizers are precisely the $Q$ functions whose softmax equals the expert policy on the expert-covered states.

The second loss is the anchor Bellman-error loss.  For any candidate $Q$, define the anchor residual
\begin{equation}
    \Delta_A(Q)(s)
    :=r_A(s)+\beta\EE[V_Q(s')\mid s,a_s]-Q(s,a_s).
    \label{eq:erm-anchor-residual}
\end{equation}
Then $Q$ satisfies the anchor Bellman equation exactly when $\Delta_A(Q)(s)=0$.  Because the expert softmax puts positive probability on every action at each covered state, this is equivalently the minimizer set of
\begin{equation}
    \EE_{(s,a)\sim d^{\pi^\star}}
    \left[
        \one\{a=a_s\}\,\Delta_A(Q)(s)^2
    \right].
    \label{eq:erm-anchor-loss}
\end{equation}
This is the main conceptual simplification: the Bellman part contains no unknown reward term, because it is evaluated only at the anchor action where $r_A(s)$ is known.

\subsection{Expected risk: the one-shot population problem}
\label{sec:erm-objective}

Combining the two losses gives the population's expected risk
\begin{equation}
\cR_{\exp}(Q)
:=
\EE_{(s,a)\sim d^{\pi^\star}}
\left[
    -\log\hat p_Q(a\mid s)
    +\one\{a=a_s\}\,
    \bigl(r_A(s)+\beta\EE[V_Q(s')\mid s,a_s]-Q(s,a_s)\bigr)^2
\right].
\label{eq:R-exp}
\end{equation}
This is the population ERM-DDC/IRL objective.  It has the flow that the classical methods suggest, but without the classical nesting:
\begin{quote}
\emph{fit the observed choices through NLL, and simultaneously fit the anchor Bellman equation through a squared residual.}
\end{quote}
The NLL term has no unknown reward in it.  The anchor Bellman term also has no unknown reward in it.  Thus, at the population level, the inverse problem has become a direct optimization over $Q$.

\begin{theorem}[Identification via expected risk minimization]
\label{thm:erm-id}
Under the anchor-action assumption and the softmax support condition, every minimizer of~\eqref{eq:R-exp} agrees with $Q^\star$ on the expert-covered state-action support.  Once $Q^\star$ is recovered, the reward is identified on that support by
\begin{equation}
    r(s,a)=Q^\star(s,a)-\beta\EE[V_{Q^\star}(s')\mid s,a].
    \label{eq:erm-pop-reward}
\end{equation}
Equivalently, if
\[
    \zeta^\star_Q(s,a):=\EE[V_Q(s')\mid s,a],
\]
then
\[
    r(s,a)=Q^\star(s,a)-\beta\zeta^\star_{Q^\star}(s,a).
\]
\end{theorem}

\begin{proof}[Proof sketch]
The NLL term is minimized exactly when $\hat p_Q(\cdot\mid s)=\pi^\star(\cdot\mid s)$ on the expert state support.  The anchor Bellman term is nonnegative and is zero at $Q^\star$.  Hence a minimizer of the sum must attain the minimum of the likelihood term and zero anchor Bellman error.  These are exactly the two equations in the Magnac--Thesmar identification argument: the likelihood pins down within-state differences, and the anchor Bellman equation pins down the state-wise level.  The reward formula is the Bellman rearrangement.
\end{proof}

\begin{remark}[The support condition is automatic for softmax experts]
\label{rem:support-auto}
For a bounded $Q^\star$ and finite action set $\cA$, the softmax policy satisfies $\pi^\star(a\mid s)>0$ for every action at every covered state.  Thus the anchor action appears with positive population probability whenever the state is covered.
\end{remark}

\subsection{Why squared TD is not the empirical Bellman error}
\label{sec:erm-td-failure}

The expected risk~\eqref{eq:R-exp} is still a population object.  The hard term is
\[
    \EE[V_Q(s')\mid s,a_s],
\]
which is an expectation over the transition kernel.  If the transition kernel were known, one could evaluate this term directly.  In the offline setting, however, we only observe samples $(s,a,s')$.

The first tempting move is to replace the conditional expectation by the one observed next state.  For a general state-action pair, define the sampled Bellman operator
\begin{equation}
    \widehat{\cT}Q(s,a,s'):=r(s,a)+\beta V_Q(s'),
    \label{eq:erm-sampled-operator}
\end{equation}
and define the sampled TD residual
\begin{equation}
    \widehat\Delta(Q)(s,a,s'):=\widehat{\cT}Q(s,a,s')-Q(s,a).
    \label{eq:erm-td-residual}
\end{equation}
At anchor actions, $r(s,a)$ is the known quantity $r_A(s)$.  The sampled operator is unbiased for the Bellman operator:
\[
    \EE[\widehat{\cT}Q(s,a,s')\mid s,a]=(\cT Q)(s,a).
\]
But the squared TD loss is not unbiased for the squared Bellman error.  The variance decomposition gives
\begin{equation}
\EE\bigl[\widehat\Delta(Q)(s,a,s')^2\mid s,a\bigr]
=
\bigl((\cT Q)(s,a)-Q(s,a)\bigr)^2
+
\beta^2\Var(V_Q(s')\mid s,a).
\label{eq:erm-td-bias}
\end{equation}
The second term is not a harmless constant: it depends on $Q$.  Therefore, minimizing squared TD error generally does \emph{not} minimize squared Bellman error.  This is the double-sampling problem in this setting.  It disappears under deterministic transitions, but not under stochastic transitions.

\subsection{Bias correction and the minimax ERM formulation}
\label{sec:erm-bias-correction}

The correction is to subtract the conditional variance term in~\eqref{eq:erm-td-bias} without estimating the transition kernel.  The variance has the least-squares representation
\begin{equation}
    \Var(V_Q(s')\mid s,a)
    =
    \min_{z\in\R}\EE[(V_Q(s')-z)^2\mid s,a],
    \qquad
    z^\star_Q(s,a)=\EE[V_Q(s')\mid s,a].
    \label{eq:erm-var-regression}
\end{equation}
Promote the scalar minimizer to a function $\zeta(s,a)$.  Combining~\eqref{eq:erm-td-bias} and~\eqref{eq:erm-var-regression} yields the corrected Bellman-error identity
\begin{align}
    \bigl((\cT Q)(s,a)-Q(s,a)\bigr)^2
    &=
    \max_{\zeta}
    \EE\left[
        \widehat\Delta(Q)(s,a,s')^2
        -\beta^2\bigl(V_Q(s')-\zeta(s,a)\bigr)^2
        \mid s,a
    \right].
    \label{eq:erm-corrected-be}
\end{align}
The maximization appears because subtracting the best least-squares error is the same as maximizing the negative squared error.

Plugging the corrected identity into the anchor-weighted expected risk gives the population minimax version of ERM-DDC/IRL:
\begin{align}
\min_Q\max_\zeta\; &
\EE_{(s,a)\sim d^{\pi^\star},\,s'\sim P(\cdot\mid s,a)}
\biggl[
    -\log\hat p_Q(a\mid s) \notag\\
&\quad
    +\one\{a=a_s\}
    \biggl\{
        \bigl(r_A(s)+\beta V_Q(s')-Q(s,a)\bigr)^2
        -\beta^2\bigl(V_Q(s')-\zeta(s,a)\bigr)^2
    \biggr\}
\biggr].
\label{eq:erm-pop-minimax}
\end{align}
The anchor indicator is important.  The Bellman-identification term is imposed at the anchor action because that is where the reward is known.  The auxiliary function $\zeta$, however, has an additional role: at the solution,
\[
    \zeta^\star_{Q^\star}(s,a)=\EE[V_{Q^\star}(s')\mid s,a],
\]
so it is exactly the continuation-value object needed to recover rewards for all supported state-action pairs:
\begin{equation}
    r(s,a)=Q^\star(s,a)-\beta\zeta^\star_{Q^\star}(s,a).
    \label{eq:erm-zeta-reward}
\end{equation}
For this reason, the paper estimates $\zeta$ over observed state-action pairs, not only at anchor observations.  The anchor-weighted corrected Bellman term identifies $Q^\star$; the all-pair conditional-mean regression makes reward recovery transition-estimation-free.

Replacing the population expectations in~\eqref{eq:erm-pop-minimax} and~\eqref{eq:erm-var-regression} by empirical averages gives the finite-sample objective.  For parametrizations $Q_\theta$ and $\zeta_\phi$, define the auxiliary regression loss
\begin{equation}
    \widehat F_\zeta(\phi;\theta)
    :=
    {1\over N}\sum_{(s,a,s')\in\cD}
    \bigl(V_{Q_\theta}(s')-\zeta_\phi(s,a)\bigr)^2.
    \label{eq:zeta-regression}
\end{equation}
For fixed $Q_\theta$, descending this loss estimates the conditional mean $\EE[V_{Q_\theta}(s')\mid s,a]$ on the observed support.  Given a current auxiliary function, the empirical $Q$ objective is
\begin{align}
\widehat f_N(\theta,\phi)
:= {1\over N}\sum_{(s,a,s')\in\cD}
\biggl[
    &-\log\hat p_{Q_\theta}(a\mid s)
    \notag\\
    &+\one\{a=a_s\}
    \biggl\{
        \bigl(r_A(s)+\beta V_{Q_\theta}(s')-Q_\theta(s,a)\bigr)^2
        -\beta^2\bigl(V_{Q_\theta}(s')-\zeta_\phi(s,a)\bigr)^2
    \biggr\}
\biggr].
\label{eq:gladius-objective}
\end{align}
Equivalently, one may view the profiled empirical risk as
\[
    \widehat g_N(\theta):=\widehat f_N(\theta,\widehat\phi(\theta)),
    \qquad
    \widehat\phi(\theta)\in\argmin_\phi \widehat F_\zeta(\phi;\theta).
\]
This is the empirical-risk version of the expected-risk problem above.

\subsection{GLADIUS: alternating gradients for the corrected ERM}
\label{sec:gladius}

GLADIUS solves the empirical problem by alternating the two natural gradient steps.  The $\zeta$ step is a regression step: fit $\zeta_\phi(s,a)$ to $V_{Q_\theta}(s')$.  The $Q$ step is a descent step on the NLL plus the anchor-weighted corrected Bellman term.

\begin{algorithm}[H]
\caption{GLADIUS: Gradient-based Learning with Ascent--Descent for Inverse Utility learning from Samples}
\label{alg:gladius}
\begin{algorithmic}[1]
\REQUIRE Dataset $\cD=\{(s,a,s')\}$, known anchor rewards $r_A(s)=r(s,a_s)$, parametrizations $Q_\theta$ and $\zeta_\phi$, stepsizes $\eta_{\theta,t},\eta_{\phi,t}$, iterations $T$.
\STATE Initialize $\theta_0,\phi_0$.
\FOR{$t=0,\ldots,T-1$}
    \STATE Draw minibatches $\cB_\zeta,\cB_Q\subseteq\cD$.
    \STATE Update $\zeta$ by descending the auxiliary regression loss:
    \[
        \phi_{t+1}
        \leftarrow
        \phi_t-
        \eta_{\phi,t}\nabla_\phi
        \left\{
        {1\over |\cB_\zeta|}\sum_{(s,a,s')\in\cB_\zeta}
        \bigl(V_{Q_{\theta_t}}(s')-\zeta_{\phi_t}(s,a)\bigr)^2
        \right\}.
    \]
    \STATE Update $Q$ by descending the corrected empirical risk:
    \[
        \theta_{t+1}
        \leftarrow
        \theta_t-
        \eta_{\theta,t}\nabla_\theta\widehat f_{\cB_Q}(\theta_t,\phi_{t+1}).
    \]
\ENDFOR
\STATE Set $\widehat Q:=Q_{\theta_T}$ and $\widehat\zeta:=\zeta_{\phi_T}$.
\STATE Recover rewards on supported state-action pairs by
\[
    \widehat r(s,a):=\widehat Q(s,a)-\beta\widehat\zeta(s,a).
\]
\RETURN $\widehat Q,\widehat\zeta,\widehat r$.
\end{algorithmic}
\end{algorithm}

The language ``ascent--descent'' comes from the minimax display~\eqref{eq:erm-pop-minimax}: maximizing the negative quadratic correction in $\zeta$ is equivalent to descending the positive regression loss~\eqref{eq:zeta-regression}.  The practical algorithm uses the regression form because it is clearer and numerically standard.

\begin{remark}[Deterministic transitions]
\label{rem:erm-deterministic}
If $s'$ is deterministic conditional on $(s,a)$, then $V_Q(s')$ has zero conditional variance.  The correction term in~\eqref{eq:erm-corrected-be} is zero, so the $\zeta$ step is unnecessary for optimizing $Q$.  In that special case, the empirical problem reduces to the NLL plus the anchor TD-squared term, and reward recovery is simply
\[
    \widehat r(s,a)=\widehat Q(s,a)-\beta V_{\widehat Q}(s'),
\]
where $s'$ is the deterministic next state from $(s,a)$.
\end{remark}

\subsection{PL geometry and convergence}
\label{sec:erm-pl}

The objective in~\eqref{eq:gladius-objective} is generally nonconvex in $\theta$.  The paper's convergence argument therefore relies on Polyak--{\L}ojasiewicz geometry, not convexity.  A differentiable function $f$ satisfies a PL inequality with constant $c>0$ on a set $U$ if
\begin{equation}
    {1\over2}\|\nabla f(x)\|_2^2
    \ge c\bigl(f(x)-f^\star\bigr),
    \qquad x\in U.
    \label{eq:pl-def}
\end{equation}
PL is weaker than convexity but still rules out spurious stationary points on the relevant set.

For the minimax objective, the relevant geometry is two-sided.  The profiled $Q$ objective must satisfy PL in $\theta$, while the auxiliary regression problem must satisfy the corresponding PL condition in $\phi$ for the ascent side (equivalently, the regression loss is PL for descent).

Before stating the conditioning assumption, it is useful to say what kind of condition it is.  We are not assuming directly that the nonconvex parameter-space objective is PL.  Instead, the argument first proves PL-type inequalities in the finite empirical output coordinates, and then transfers them to the parameters through the Jacobians of the output maps
\[
    \theta\mapsto (Q_\theta(z))_{z\in Z_Q},
    \qquad
    \phi\mapsto (\zeta_\phi(\bar z))_{\bar z\in Z_\zeta}.
\]
This transfer condition is checkable for common parametrizations.  For linear models it is exactly a full-row-rank empirical feature condition.  For sufficiently wide smooth neural networks, the same type of lower bound is obtained in the usual NTK/lazy-training regime on any fixed finite empirical evaluation set.  We state the assumption first and then record these two verification cases as lemmas.

\begin{assumption}[Empirical-output Jacobian conditioning]
\label{ass:erm-jac}
Let $Z_Q(\cD)$ be the finite set of empirical state-action points at which $Q_\theta$ is evaluated, and let $Z_\zeta(\cD)$ be the finite set at which $\zeta_\phi$ is evaluated.  On finite-radius neighborhoods $B_Q$ and $B_\zeta$ containing the iterates and empirical minimizers:
\begin{enumerate}[leftmargin=2em,label=(\roman*)]
\item $Q_\theta(z)$ and $\zeta_\phi(\bar z)$ are twice continuously differentiable with bounded first and second derivatives on the empirical evaluation sets.
\item The empirical output Jacobian of $Q_\theta$ is uniformly well conditioned:
\[
    J_Q(\theta;Z_Q)J_Q(\theta;Z_Q)^\top\succeq \mu_Q I
    \qquad\text{for all }\theta\in B_Q.
\]
\item The empirical output Jacobian of $\zeta_\phi$ is uniformly well conditioned:
\[
    J_\zeta(\phi;Z_\zeta)J_\zeta(\phi;Z_\zeta)^\top\succeq \mu_\zeta I
    \qquad\text{for all }\phi\in B_\zeta.
\]
\end{enumerate}
\end{assumption}

\begin{lemma}[Linear parametrizations satisfy Assumption~\ref{ass:erm-jac}]
\label{lem:erm-jac-linear}
Suppose
\[
    Q_\theta(z)=\theta^\top\varphi(z),
    \qquad
    \zeta_\phi(\bar z)=\phi^\top\chi(\bar z).
\]
Let $\Phi_Q$ be the matrix with rows $\varphi(z)^\top$ for $z\in Z_Q(\cD)$, and let $X_\zeta$ be the matrix with rows $\chi(\bar z)^\top$ for $\bar z\in Z_\zeta(\cD)$.  If
\[
    \Phi_Q\Phi_Q^\top\succeq \mu_Q I,
    \qquad
    X_\zeta X_\zeta^\top\succeq \mu_\zeta I
\]
for some $\mu_Q,\mu_\zeta>0$, and the feature evaluations on these finite empirical sets are bounded, then Assumption~\ref{ass:erm-jac} holds.
\end{lemma}

\begin{proof}[Proof sketch]
For the linear $Q$ class,
\[
    (Q_\theta(z))_{z\in Z_Q}=\Phi_Q\theta,
    \qquad
    J_Q(\theta;Z_Q)=\Phi_Q,
\]
so the assumed Gram lower bound is exactly Assumption~\ref{ass:erm-jac}(ii).  The first derivative is $\varphi(z)$ and the second derivative is zero, so the smoothness and bounded-derivative part follows from bounded feature evaluations on the finite set $Z_Q$.  The same calculation with $X_\zeta$ gives Assumption~\ref{ass:erm-jac}(iii) and the corresponding derivative bounds for $\zeta_\phi$.
\end{proof}

\begin{lemma}[Wide smooth neural networks satisfy Assumption~\ref{ass:erm-jac}]
\label{lem:erm-jac-nn}
Fix the empirical evaluation sets $Z_Q(\cD)$ and $Z_\zeta(\cD)$.  Consider sufficiently wide feedforward neural networks with $C^2$ activations and standard over-parameterized random initialization.  If the neural tangent Gram matrices at initialization on $Z_Q$ and $Z_\zeta$ have strictly positive minimum eigenvalues, and the SGDA iterates remain in a lazy-training neighborhood where the empirical output Jacobians move by less than half of these eigenvalue margins, then Assumption~\ref{ass:erm-jac} holds with high probability.
\end{lemma}

\begin{proof}[Proof sketch]
On a fixed finite evaluation set, standard over-parameterized initialization gives well-conditioned empirical tangent kernels with high probability under nondegenerate data.  Because the networks use $C^2$ activations and the iterates stay in a finite-radius lazy-training neighborhood, first and second derivatives are uniformly bounded on the finite empirical sets.  The small-movement condition implies the empirical Jacobian Gram matrices along the optimization path remain close to their initialization values.  By Weyl's inequality, losing at most half the initial spectral margin preserves positive lower bounds $\mu_Q$ and $\mu_\zeta$ throughout the balls $B_Q$ and $B_\zeta$.
\end{proof}

\begin{theorem}[Empirical PL guarantees]
\label{thm:erm-pl}
Under Assumption~\ref{ass:erm-jac}, the empirical $Q$ objective obtained after profiling the auxiliary correction satisfies a PL inequality on $B_Q$.  For each fixed $\theta\in B_Q$, the auxiliary regression loss $\widehat F_\zeta(\cdot;\theta)$ also satisfies a PL inequality on $B_\zeta$; equivalently, $-\widehat F_\zeta(\cdot;\theta)$ satisfies the ascent-side PL condition.
\end{theorem}

\begin{proof}[Proof sketch]
The proof is not the generic claim that sums of PL functions are always PL; that claim is false.  Here the two pieces control complementary directions.  The NLL term controls within-state action-logit directions, while the anchor Bellman term controls the state-wise constant directions left invisible by the softmax likelihood.  The empirical-output Jacobian lower bounds transfer these output-space inequalities to the parameter space.  The auxiliary $\zeta$ problem is a quadratic regression in empirical $\zeta$ outputs, so it has the corresponding PL geometry after the same Jacobian-conditioning transfer.
\end{proof}

\begin{theorem}[GLADIUS convergence and population excess risk]
\label{thm:gladius-conv}
Suppose the realizability and empirical-output Jacobian-conditioning assumptions hold, and suppose the stochastic gradients use stepsizes $\eta_t=c_1/(c_2+t)$ satisfying the two-sided PL SGDA conditions.  Then the GLADIUS iterates satisfy
\begin{equation}
    \EE\!
    \left[
        \widehat g_N(\widehat\theta_T)-\widehat g_N^\star
    \right]
    \le {\nu\over \gamma_0+T}.
    \label{eq:gladius-empirical-rate}
\end{equation}
Moreover, combining the empirical optimization bound with the paper's stability/generalization argument yields
\begin{equation}
    \EE\!\left[
        \cR_{\exp}(Q_{\widehat\theta_T})-
        \cR_{\exp}(Q^\star)
    \right]
    \le
    (1+L/\rho)G\,\varepsilon_{N,T}
    +{\nu\over \gamma_0+T},
    \label{eq:gladius-pop-risk-rate}
\end{equation}
where
\[
    \varepsilon_{N,T}
    =O\!\left((c_2+T)^{-\alpha}\right)+{C\over N},
    \qquad
    \alpha:=\min\left\{{1\over2},{3cc_1\over8}\right\}.
\]
If the quantitative identification inequality for the full ERM-DDC/IRL risk also holds, then
\[
    \EE\left[\|Q_{\widehat\theta_T}-Q^\star\|_{2,*}^2\right]
    =
    O\!\left((c_2+T)^{-\alpha}\right)+O\!\left({1\over N}\right)+O\!\left({1\over T}\right).
\]
Thus the $L_2(d^\star)$ estimation error has statistical contribution $O(N^{-1/2})$ and, when $\alpha=1/2$, optimization contribution $O(T^{-1/4})$.
\end{theorem}

The main takeaway is now straightforward.  ERM-DDC/IRL is a transition-estimation-free minimax estimator for $Q^\star$.  It is built by translating the two classical identifying equations into losses, correcting the empirical Bellman term for double-sampling bias, and recovering rewards from the same learned conditional-mean correction $\widehat\zeta$.

\section*{Closing Remarks}

The classical DDC literature gave us identification through the anchor-action argument, plus computational templates such as NFXP and CCP.  The modern IRL literature gave us scalable objectives, but often without pointwise Bellman identification of the reward.  The ERM framework of \citet{kang2025empirical} combines the useful pieces from both sides: use the likelihood term for within-state action-value differences, use the anchor Bellman equation for the state-wise level, use the bias-corrected TD identity to avoid transition-kernel estimation, and solve the resulting objective by alternating gradients.  This is why the section should be read as a build-up from classical DDC to a one-shot gradient-based estimator, rather than as an isolated minimax formula.

\appendix

\section{Deferred Proofs}
\label{app:deferred-proofs}

\subsection{Local AIRL Gradient and Policy-Objective Identities}
\label{app:airl-local-identities}

For completeness, we give the precise finite-horizon version of Lemma~\ref{lem:airl-trajectory-to-transition}. Work on a finite horizon $H$ and suppress discounting; the discounted version inserts the weight $\beta^t$ in every time-$t$ term. Write a transition as $x_t := (s_t,a_t,s_{t+1})$. For any policy $\pi$, write
\begin{equation*}
\rho_t^\pi(s,a,s')
\;:=\;
\PP_{\nu_0,\pi}(s_t=s)\,\pi(a\mid s)\,P(s'\mid s,a)
\end{equation*}
for the time-$t$ transition marginal, and define the policy's pre-action transition base
\begin{equation*}
b_t^\pi(s,a,s')
\;:=\;
\PP_{\nu_0,\pi}(s_t=s)\,P(s'\mid s,a),
\qquad
\rho_t^\pi(s,a,s')=b_t^\pi(s,a,s')\pi(a\mid s).
\end{equation*}
For a differentiable transition score $f_\omega(s,a,s')$, define the finite-horizon maximum-causal-entropy trajectory model
\begin{equation}
p_\omega(\tau)
\;:=\;
\frac{1}{Z_\omega}\,\nu_0(s_0)\prod_{t=0}^{H-1}P(s_{t+1}\mid s_t,a_t)
\exp\!\Bigl(\sum_{t=0}^{H-1}f_\omega(s_t,a_t,s_{t+1})\Bigr),
\label{eq:airl-mce-trajectory-model}
\end{equation}
and let $p_{\omega,t}$ denote its transition marginal at time $t$. The maximum-causal-entropy log-likelihood is
\begin{equation*}
\cJ_{\mathrm{MCE}}(\omega)
\;:=\;
\EE_{\nu_0,\pi^\star}[\log p_\omega(\tau)].
\end{equation*}
For any sampler $\mu_t$ whose support contains that of $p_{\omega,t}$, the same gradient can be written in importance-sampled form. AIRL's local approximation replaces the intractable marginal $p_{\omega,t}$ by the policy-based energy measure
\begin{equation}
\widehat p_{\omega,t}^{\pi}(s,a,s')
\;:=\;
\exp(f_\omega(s,a,s'))\,b_t^\pi(s,a,s'),
\label{eq:airl-local-energy-marginal}
\end{equation}
and compares it with the policy transition marginal through
\begin{equation}
\widehat\mu_{\omega,t}^{\pi}(s,a,s')
\;:=\;
\tfrac12\widehat p_{\omega,t}^{\pi}(s,a,s')+\tfrac12\rho_t^\pi(s,a,s').
\label{eq:airl-local-mixture}
\end{equation}
The following proposition is the rigorous version of the localization statement used in the main text.

\begin{proposition}[Finite-horizon local AIRL surrogate]
\label{prop:airl-local-identities}
Fix the current policy $\pi$ and a differentiable score $f_\omega$. The maximum-causal-entropy gradient is
\begin{equation}
\nabla_\omega \cJ_{\mathrm{MCE}}(\omega)
\;=
\sum_{t=0}^{H-1}
\left\{
\EE_{x\sim \rho_t^{\pi^\star}}[\nabla_\omega f_\omega(x)]
-
\EE_{x\sim p_{\omega,t}}[\nabla_\omega f_\omega(x)]
\right\},
\label{eq:airl-mce-gradient}
\end{equation}
or equivalently
\begin{equation}
\nabla_\omega \cJ_{\mathrm{MCE}}(\omega)
\;=
\sum_{t=0}^{H-1}
\left\{
\EE_{\rho_t^{\pi^\star}}[\nabla_\omega f_\omega]
-
\EE_{\mu_t}\!\left[\frac{p_{\omega,t}(x)}{\mu_t(x)}\nabla_\omega f_\omega(x)\right]
\right\}.
\label{eq:airl-importance-mce-gradient}
\end{equation}
For each time $t$, let
\begin{equation*}
\bar\mu_t^\pi(s,a,s'):=\tfrac12\rho_t^{\pi^\star}(s,a,s')+\tfrac12\rho_t^\pi(s,a,s')
\end{equation*}
be the actual half-expert, half-policy transition-sampling mixture used by the discriminator. Define the transition discriminator $D_{f_\omega}^{\pi}$ by~\eqref{eq:airl-log-odds}, equivalently by~\eqref{eq:airl-generic-discriminator}, and define
\begin{equation}
\cJ_{\mathrm{disc}}(\omega;\pi)
\;:=
\sum_{t=0}^{H-1}
\left\{
\EE_{x\sim\rho_t^{\pi^\star}}[\log D_{f_\omega}^{\pi}(x)]
+
\EE_{x\sim\rho_t^\pi}[\log(1-D_{f_\omega}^{\pi}(x))]
\right\}.
\label{eq:airl-local-disc-objective}
\end{equation}
Then
\begin{equation}
\nabla_\omega \cJ_{\mathrm{disc}}(\omega;\pi)
\;=
\sum_{t=0}^{H-1}
\left\{
\EE_{\rho_t^{\pi^\star}}[\nabla_\omega f_\omega]
-
\EE_{\bar\mu_t^\pi}\!\left[
\frac{\widehat p_{\omega,t}^{\pi}(x)}{\widehat\mu_{\omega,t}^{\pi}(x)}
\nabla_\omega f_\omega(x)
\right]
\right\}.
\label{eq:airl-disc-gradient-identity}
\end{equation}
Thus, under the local sampler-consistency approximation $\bar\mu_t^\pi=\widehat\mu_{\omega,t}^{\pi}$, the discriminator gradient is exactly the importance-sampled gradient obtained from~\eqref{eq:airl-importance-mce-gradient} after replacing $p_{\omega,t}$ by $\widehat p_{\omega,t}^{\pi}$. If the adapted policy sampler is consistent in the stronger sense that $\widehat p_{\omega,t}^{\pi}=p_{\omega,t}$, this is the maximum-causal-entropy IRL gradient~\eqref{eq:airl-mce-gradient}.

Moreover, the discriminator log-odds reward
\begin{equation*}
\widehat r_f^{\pi}(s,a,s')
\;:=\;
\log D_f^{\pi}(s,a,s')-\log(1-D_f^{\pi}(s,a,s'))
\end{equation*}
satisfies
\begin{equation}
\widehat r_f^{\pi}(s,a,s')
\;=
\;f(s,a,s')-\log\pi(a\mid s),
\label{eq:airl-log-odds-reward}
\end{equation}
and therefore
\begin{equation}
\EE_{\nu_0,\pi}\!\left[\sum_{t=0}^{H-1}\widehat r_f^{\pi}(s_t,a_t,s_{t+1})\right]
\;=
\EE_{\nu_0,\pi}\!\left[\sum_{t=0}^{H-1}
\bigl(f(s_t,a_t,s_{t+1})+\Ent(\pi(\cdot\mid s_t))\bigr)\right].
\label{eq:airl-policy-objective-identity}
\end{equation}
So optimizing the policy against the log-odds reward is exactly entropy-regularized policy optimization with reward $f$.
\end{proposition}

\begin{proof}[Proof of Proposition~\ref{prop:airl-local-identities}]
First derive the maximum-causal-entropy gradient. From~\eqref{eq:airl-mce-trajectory-model},
\begin{equation*}
\log p_\omega(\tau)
=
-\log Z_\omega+\text{terms independent of }\omega
+
\sum_{t=0}^{H-1}f_\omega(x_t).
\end{equation*}
Therefore
\begin{equation*}
\nabla_\omega\cJ_{\mathrm{MCE}}(\omega)
=
\sum_{t=0}^{H-1}\EE_{\rho_t^{\pi^\star}}[\nabla_\omega f_\omega(x)]
-
\nabla_\omega\log Z_\omega.
\end{equation*}
Differentiating the normalizer gives
\begin{equation*}
\nabla_\omega\log Z_\omega
=
\EE_{\tau\sim p_\omega}\!\left[\sum_{t=0}^{H-1}\nabla_\omega f_\omega(x_t)\right]
=
\sum_{t=0}^{H-1}\EE_{x\sim p_{\omega,t}}[\nabla_\omega f_\omega(x)],
\end{equation*}
which proves~\eqref{eq:airl-mce-gradient}. The importance-sampled form~\eqref{eq:airl-importance-mce-gradient} follows by writing
\begin{equation*}
\EE_{x\sim p_{\omega,t}}[\nabla_\omega f_\omega(x)]
=
\EE_{x\sim\mu_t}\!\left[\frac{p_{\omega,t}(x)}{\mu_t(x)}\nabla_\omega f_\omega(x)\right]
\end{equation*}
whenever $\mu_t$ dominates $p_{\omega,t}$.

We now prove the local discriminator identity. From~\eqref{eq:airl-generic-discriminator},
\begin{equation*}
\log D_{f_\omega}^{\pi}(x)
= f_\omega(x)-\log\bigl(\exp f_\omega(x)+\pi(a\mid s)\bigr),
\end{equation*}
and
\begin{equation*}
\log(1-D_{f_\omega}^{\pi}(x))
= \log\pi(a\mid s)-\log\bigl(\exp f_\omega(x)+\pi(a\mid s)\bigr).
\end{equation*}
Substituting these two identities into~\eqref{eq:airl-local-disc-objective} yields
\begin{align*}
\cJ_{\mathrm{disc}}(\omega;\pi)
&=\sum_{t=0}^{H-1}\Bigl\{
\EE_{\rho_t^{\pi^\star}}[f_\omega(x)]
+\EE_{\rho_t^\pi}[\log\pi(a\mid s)] \\
&\hspace{4em}
-\EE_{\rho_t^{\pi^\star}}\!\left[\log\bigl(\exp f_\omega(x)+\pi(a\mid s)\bigr)\right]
-\EE_{\rho_t^\pi}\!\left[\log\bigl(\exp f_\omega(x)+\pi(a\mid s)\bigr)\right]
\Bigr\}.
\end{align*}
The second term is independent of $\omega$. Differentiating the remaining terms gives
\begin{align*}
\nabla_\omega\cJ_{\mathrm{disc}}(\omega;\pi)
&=\sum_{t=0}^{H-1}\left\{
\EE_{\rho_t^{\pi^\star}}[\nabla_\omega f_\omega(x)]
-2\EE_{\bar\mu_t^\pi}\!\left[
\frac{\exp f_\omega(x)}{\exp f_\omega(x)+\pi(a\mid s)}\nabla_\omega f_\omega(x)
\right]
\right\}\\
&=\sum_{t=0}^{H-1}\left\{
\EE_{\rho_t^{\pi^\star}}[\nabla_\omega f_\omega(x)]
-\EE_{\bar\mu_t^\pi}\!\left[
\frac{\exp f_\omega(x)}{\tfrac12\exp f_\omega(x)+\tfrac12\pi(a\mid s)}\nabla_\omega f_\omega(x)
\right]
\right\}.
\end{align*}
Multiplying the numerator and denominator of the fraction by $b_t^\pi(s,a,s')$, and using $\widehat p_{\omega,t}^{\pi}=\exp(f_\omega)b_t^\pi$ and $\rho_t^\pi=\pi(a\mid s)b_t^\pi$, gives
\begin{equation*}
\frac{\exp f_\omega(x)}{\tfrac12\exp f_\omega(x)+\tfrac12\pi(a\mid s)}
=
\frac{\widehat p_{\omega,t}^{\pi}(x)}{\tfrac12\widehat p_{\omega,t}^{\pi}(x)+\tfrac12\rho_t^\pi(x)}
=
\frac{\widehat p_{\omega,t}^{\pi}(x)}{\widehat\mu_{\omega,t}^{\pi}(x)},
\end{equation*}
which proves~\eqref{eq:airl-disc-gradient-identity}. The two sampler-consistency consequences follow immediately by comparing~\eqref{eq:airl-disc-gradient-identity} with~\eqref{eq:airl-importance-mce-gradient} and~\eqref{eq:airl-mce-gradient}.

Finally, divide~\eqref{eq:airl-generic-discriminator} by
\begin{equation*}
1-D_f^\pi(s,a,s')=\frac{\pi(a\mid s)}{\exp f(s,a,s')+\pi(a\mid s)}.
\end{equation*}
This gives~\eqref{eq:airl-log-odds-reward}. Summing over a trajectory and taking expectation under $\pi$,
\begin{align*}
\EE_{\nu_0,\pi}\!\left[\sum_{t=0}^{H-1}\widehat r_f^\pi(s_t,a_t,s_{t+1})\right]
&=\EE_{\nu_0,\pi}\!\left[\sum_{t=0}^{H-1}f(s_t,a_t,s_{t+1})-\log\pi(a_t\mid s_t)\right]\\
&=\EE_{\nu_0,\pi}\!\left[\sum_{t=0}^{H-1}f(s_t,a_t,s_{t+1})+\Ent(\pi(\cdot\mid s_t))\right],
\end{align*}
where the last equality uses $\EE_{a_t\sim\pi(\cdot\mid s_t)}[-\log\pi(a_t\mid s_t)]=\Ent(\pi(\cdot\mid s_t))$.
\end{proof}

\nocite{*}
\printbibliography[title={References}]

\end{document}